\documentclass[twoside,11pt,preprint]{article}

\usepackage{jmlr2e}

\usepackage{amsmath,amssymb}
\usepackage{bbm}
\usepackage{mathtools}
\usepackage{algorithm}
\usepackage{algorithmic}
\usepackage{color}

\newcommand{\modifyA}[1]{{#1}}
\newcommand{\argmax}{\mathop{\rm argmax}\limits}
\newcommand{\erfc}{~{\rm erfc}}

\jmlrheading{XX}{2022}{XX-XX}{XX/XX}{XX/XX}{XXXXXX}{Nobuaki Kikkawa and Hiroshi Ohno}

\ShortHeadings{Materials Discovery using Max $K$-Armed Bandit}{Kikkawa and Ohno}
\firstpageno{1}

\begin{document}

\title{Materials Discovery using Max $K$-Armed Bandit}

\author{\name Nobuaki Kikkawa \email kikkawa@mosk.tytlabs.co.jp \\
       \addr Toyota Central R{\rm \&}D Labs., Inc.\\
       41-1, Yokomichi, Nagakute, Aichi 480-1192, Japan
       \AND
       \name Hiroshi Ohno \email oono-h@mosk.tytlabs.co.jp \\
       \addr Toyota Central R{\rm \&}D Labs., Inc.\\
       41-1, Yokomichi, Nagakute, Aichi 480-1192, Japan}

\editor{XXXX XXXX}

\maketitle

\begin{abstract}
  Search algorithms for the bandit problems are applicable in materials discovery.
  However, the objectives of the conventional bandit problem
  are different from those of materials discovery.
  The conventional bandit problem aims to maximize the total rewards,
  whereas materials discovery aims to achieve breakthroughs in material properties.
  The max $K$-armed bandit (MKB) problem, which aims to acquire the single best reward,
  matches with the discovery tasks better than the conventional bandit.
  Thus, here, we propose a search algorithm for materials discovery
  based on the MKB problem using a pseudo-value of
  the upper confidence bound \modifyA{of expected improvement of the best reward.
  This approach is pseudo-guaranteed to be asymptotic oracles
  that do not depends on the time horizon. 
  In addition,} compared with other MKB algorithms,
  the proposed algorithm has only one hyperparameter,
  which is advantageous in materials discovery.
  We applied the proposed algorithm to synthetic problems
  and molecular-design demonstrations using a Monte Carlo tree search.
  According to the results, the proposed algorithm stably outperformed
  other bandit algorithms in the late stage of the search process 
  when the optimal arm of the MKB could not be determined based on its expectation reward.
\end{abstract}

\begin{keywords}
  Max $K$-armed bandit problem, Confidence bounds,
  Monte Carlo tree search, Molecular design, \modifyA{greedy oracle}
\end{keywords}

\section{Introduction}\label{sec1}
Materials discovery integrated with machine learning
is a field with immense growth potential.
Material property predictions using regression and clustering methods
\citep{LIU2017159,C8ME00012C,butler2018machine,
ramprasad2017machine,pilania2013accelerating}
are recognized as a beneficial approach in the development workplace.
Materials discovery using deep learning \citep{agrawal2019deep,jha2018elemnet},
transfer learning \citep{jha2019enhancing,yamada2019predicting},
and generative models \citep{sanchez2017optimizing,sanchez2018inverse}
is actively under investigation in advanced researches.
Autonomous searches based on Bayesian optimization
\citep{ueno2016combo,kusne2020fly},
Monte Carlo tree search (MCTS) 
\citep{m2017mdts,yang2017chemts,ju2018optimizing,segler2018planning,
kiyohara2018searching,m2018structure,kajita2020autonomous,
kikkawa2020self,patra2020accelerating},
and reinforcement learning (RL)
\citep{sanchez2017optimizing,popova2018deep,olivecrona2017molecular}
have also been investigated to accelerate materials discovery.
Active learning approaches \citep{kusne2020fly,del2020assessing}
and the effective use of failed experiments \citep{raccuglia2016machine}
are also important for overcoming the limitations of data generation in materials science.

Finding novel materials with record-breaking properties of interest
is one of the goals of materials discovery.
However, the guiding principles of MCTS and RL
appear to differ from the goal of materials discovery
because these approaches mainly focus on maximizing the total reward
\citep{auer2002finite,kocsis2006bandit,browne2012survey,sutton2018reinforcement}
rather than discovering a record-breaking material property.
Therefore, these approaches tend to avoid selections with high failure rates
even though those could lead to a few great breakthroughs.
Because failure is often a prerequisite for success,
these approaches are not always optimal for achiving a significant discovery.

The max $K$-armed bandit (MKB) problem \citep{cicirello2005max},
also called the extreme bandit \citep{carpentier2014extreme}
or the max bandit \citep{david2016pac},
is a promising problem setting for materials discovery.
In the MKB problem, a player aims to maximize the single best reward
from a slot with $K$ arms instead of the total reward
in the conventional bandit problem \citep{lai1985asymptotically}.
Owing to these modifications, the algorithms for the MKB problem
can explore the adventurous arm rather than the stable arm
\citep{carpentier2014extreme,streeter2006simple,achab2017max}.

Several algorithms have been proposed for the MKB problem
\citep{carpentier2014extreme,david2016pac,streeter2006simple,
achab2017max,streeter2006asymptotically}.
However, their practical applications in materials discovery are limited.
Some of them consider the \modifyA{time horizon} $T$ as a hyperparameter,
even though their applications for MCTS are associated with many drawbacks.
Other methods involve many hyperparameters depending on unknown reward distributions.
This requires a time-consuming parameter tuning,
which is extremely costly for materials discovery.
To overcome these difficulties, we propose a \modifyA{MKB} algorithm with one hyperparameter
that employs a pseudo-value of upper confidence bound (UCB)
\modifyA{of the expected improvement (EI) of the maximum reward}
as the selection index of the arm.
We apply this algorithm to synthetic problems
and demonstrations of materials discovery using MCTS.

The primary contributions of this study are as follows:
\begin{enumerate}
\item We propose a MKB algorithm by introducing a pseudo-value
of the UCB of \modifyA{EI of the} best reward,
which has only one hyper-parameter.
\item We demonstrate the effectiveness of the proposed algorithm
for materials discovery using MCTS and compare the results
with those obtained using other bandit algorithms.
\item \modifyA{We prove that asymptotically optimal MKB algorithms
can be generated using UCB of EI of the best reward.}
\item \modifyA{We propose a time-independent oracle named {\it Kikkawa's greedy oracle}.
This oracle makes it possible to discuss the MKB problem
in the almost same manner as the conventional bandit problem.}
\item To the best of our knowledge,
this is the first study to actually apply the MKB algorithm to materials discovery.
\end{enumerate}

The remainder of this paper is structured as follows.
Section \ref{sec2} describes the related work of MKB,
materials discovery using MCTS, and related algorithms.
After that, we define some terms and representations in Section \ref{sec3}.
In Section \ref{sec4}, we describe the \modifyA{idea to create a MKB algorithm}.
In the following Section \ref{sec5},
the derivation of the proposed algorithm is presented.
Section \ref{sec6} demonstrates the experiments
conducted for comparing the proposed algorithm with other bandit algorithms.
\modifyA{In Section \ref{sec7}, we discuss the subtleties of the MKB problem
and the theoretical aspect of our idea}.
Finally, Section \ref{sec8} presents our conclusions
and discussions on the future outlook of the proposed algorithm.

\section{Related Work}\label{sec2}
In this section, we first describe the related work of MKB,
materials discovery using MCTS, and related algorithms.

\subsection{Max $K$-armd bandit problem}
The MKB problem is expressed as a policy-decision problem
that maximizes the single best reward $\mathop{\max}_{t\in[T]} r_{k(t)}(t)$,
where $[T] \coloneqq \{1,2, \ldots ,T\}$, $r_{k(t)}$
is the reward from the $k$-th arm at time $t$
with a time-independent distribution $f_{k}(r)$,
and $k(t)$ is the selected arm index at time $t$
determined based on the policy to be tuned \citep{cicirello2005max}.
This is a simple variant of the conventional bandit problem,
which aims to maximize the total reward $\sum_{t\in[T]} r_{k(t)}(t)$
\citep{lai1985asymptotically}.

The MKB problem was first proposed by \citet{cicirello2005max},
who derived the optimal allocation order for the Gumbel-type reward distribution.
The following year, \citet{streeter2006asymptotically}
proposed an asymptotically optimal algorithm using the explore-then-commit (ETC) approach.
They also proposed a UCB algorithm for the MKB problem, called {\it ThresholdAscent},
in the same year \citep{streeter2006simple}.
This algorithm used a UCB of $\mathbb{E}[\mathbbm{1}[r_{k}>r^{s\text{-th}}]]$
as the selection index, where $r^{s\text{-th}}$
was the $s$-th maximum of observed rewards.

The next stream of the algorithm development for the MKB problem
was undertaken by \citet{carpentier2014extreme}.
They estimated a finite-time upper bound of $\mathbb{E}[r^{\max}]$
assuming the reward distribution as the second-order Pareto distribution
and proposed {\it ExtremeHunter} algorithm based on it.
The ETC version of {\it ExtremeHunter} was proposed by \citet{achab2017max},
and they also proposed a simple algorithm, denoted {\it RobustUCBMax} in this paper.
The {\it RobustUCBMax} used a robust UCB \citep{bubeck2013bandits}
of $\mathbb{E}[r_k\mathbbm{1}[r_{k}>u]]$, where $u$ was a threshold parameter.
A probably approximately correct (PAC) approach for the MKB problem called {\it Max-CB}
was discussed theoretically by \citet{david2016pac}.
Here, we summarize the features of the MKB algorithms in Table \ref{maxBandit}.
Additionally, we show the main target of these algorithms in this table.
The MKB algorithms have not been applied for materials discovery in the previous studies,
although it was discussed by \citet{david2016pac}.
\begin{table}[htbp]
  \caption{
    MKB algorithms.
    The term ``Anytime'' refers to algorithms that do not employ $T$ as a hyperparameter.
  } \label{maxBandit}
  \begin{tabular}{llrrl}
    \hline
    Algorithm & \multicolumn{1}{l}{Approach}
    & \multicolumn{1}{l}{Anytime} & \multicolumn{1}{l}{\# of parameters} 
    & \multicolumn{1}{l}{Target}\\
    \hline
    {\it MaxSearch} (this work) & pseudo-UCB & yes & 1 & 
      \raise 2pt \vtop{\hbox{\lower 2pt \vtop{
        \hbox{synthetic problem,}\hbox{materials discovery}
      }}}\\
    \raise 2pt \vtop{\hbox{\lower 2pt \vtop{
      \hbox{asymptotically}\hbox{algorithm}
    }}} & ETC & no & 2 & -\\
    {\it ThresholdAscent} & UCB & no & 2 & scheduling\\
    {\it ExtremeHunter} &
      \raise 2pt \vtop{\hbox{\lower 2pt \vtop{
        \hbox{finite-time}\hbox{upper bound}
      }}} & no & $>1$ & \raise 2pt \vtop{\hbox{\lower 2pt \vtop{
        \hbox{synthetic problems,}\hbox{traffic analysis}
      }}}\\
    {\it ExtremeETC} & ETC & no & $>1$ & synthetic problem\\
    {\it RobustUCBMax} & UCB & yes & 3 & synthetic problem\\
    {\it Max-CB} & PAC & yes & 2 & -\\
    \hline
  \end{tabular}
\end{table}

\modifyA{
Some researchers also contributed theoretically.
\citet{cicirello2005max} stated that in the case of the MKB problem
with the Gumbel-type reward distributions,
the optimal algorithm should sample the observed best arm at a rate
increasing double exponentially relative to the other arms.
\citet{carpentier2014extreme} introduced an expected regret,
called {\it extreme regret}, for the MKB problem.
They also proposed an algorithm where the regret had
$o(\mathbb{E}\left[\max_{t\in [T]}r_k(t)\right])$.
Although those theoretical progresses were traced
to the analogies of the conventional bandit problem,
\citet{nishihara2016no} proved that no policy
is guaranteed to asymptotically approach the oracle used by \citet{carpentier2014extreme}.
\citet{nishihara2016no} also pointed out some other subtleties on the MKB problem
and proposed an oracle using EI although they does not analyze it much.
}

\subsection{Materials discovery using Monte Carlo tree search}
There are several studies relating to materials discovery using MCTS
\citep{m2017mdts,yang2017chemts,ju2018optimizing,segler2018planning,
kiyohara2018searching,m2018structure,kajita2020autonomous,
kikkawa2020self,patra2020accelerating}.
\citet{m2017mdts}, in the pioneering work,
compiled Si-Ge interfacial conformations into binaries
and optimized them to maximize the thermal conductance using MCTS.
\citet{ju2018optimizing} also optimized
the interface roughness by ternary embedding.
The optimizations of the grain boundary \citep{kiyohara2018searching},
doping \citep{m2018structure}, and chemical syntheses
\citep{segler2018planning,patra2020accelerating}
have also been investigated.

\citet{yang2017chemts} applied MCTS
to the optimization of chemical structures.
They introduced a search tree in which nodes correspond
to the simplified molecular-input line-entry system (SMILES) characters
\citep{weininger}, e.g., ``C'' ``O'', ``('', and ``)''.
Because the SMILES grammer can express most of molecules, 
the chemical-structure optimization is regarded as a string optimization in this approach.
They showed that the MCTS approach outperformed other approaches in the SMILES search.

The MCTS approach using SMILES was employed in subsequent studies.
\citet{kajita2020autonomous} introduced fragments of SMILES, such as ``CC'' and ``CO''
to restrict the search space of chemical structures.
In their study, they attempted 5,500 evaluations 10 times
using molecular dynamics (MD) simulations in a search run.
They also confirmed the properties of the molecules
with high rewards by synthetic experiments.
\citet{kikkawa2020self} improved the flexibility of the restriction
by introducing rule-based grammar into the search tree using a maze game.
They also evaluated several thousands molecules in a search run using MD simulations.

\subsection{Other algorithms}
The application of single-player MCTS \citep{schadd2008single}
for materials discovery has also been considered.
In this approach, a variance-dependent term
is empirically added to the selection index of the UCB.
Herein, we denote the bandit algorithm using this modified index {\it spUCB}. 

We note that the best-arm identification,
such as the {\it UCBE} algorithm \citep{audibert2010best},
is different from the MKB algorithm.
The best-arm identification aims to find the arm
with the maximum ``expectation''reward
not the ``single'' maximum through a search run.
The algorithms based on the best-arm identification
barely select arms with a low expectation reward
even if the arm affords a high reward at low rates.

\modifyA{
\section{Definitions}\label{sec3}
The definitions used in this section through Section \ref{sec5} are listed as follows:
\begin{definition}[Bandit problem]
The $K$-armed bandit problem, or simply bandit problem,
is a problem to maximize (minimize) some objective
\begin{equation}
G\left[\left\{k(t)\right\}_{t\in [T]};\left\{r_k(t)\right\}_{k\in [K],t\in [T]}\right]
\end{equation}
in a selection game with $K$ arms during time horizon $T$,
where $k(t),t\in[T]$ is a player's selections which should be optimized.
The arm $k\in[K]$ returns a reward $r_k(t), t\in[T]$ at time $t$, 
following unknown time-independent reward distirbution $f_k(r)$.
A player also does not know $T$ in the "anytime" setting.
We usually omit the dependency of $G$ on $k(t)$ and $r_k(t), k\in[K], t\in[T]$.
\end{definition}
\begin{definition}[Conventional bandit problem]
The conventional bandit problem is a bandit problem to maximize the total reward
\begin{equation}
  G^\text{sum}\left[\left\{k(t)\right\}_{t\in [T]}\right]\coloneqq\sum_{t\in[T]}r_{k(t)}(t).
\end{equation}
\end{definition}
\begin{definition}[MKB problem]
The MKB problem is a bandit problem to maximize the sigle maximum reward
\begin{equation}
  G^\text{max}\left[\left\{k(t)\right\}_{t\in [T]}\right]\coloneqq\max_{t\in[T]}r_{k(t)}(t).
\end{equation}
\end{definition}
\begin{definition}[EI]
  In the bandit problem, the EI of arm $k$ at time $\tau\leq T$ is defined as
  \begin{equation}
    EI\left[k,\tau;G\right]
    \coloneqq \mathbb{E}_{f_k}\left[G\left[
      \{\tilde{k}(t)\}_{t\in[\tau]}
    \right]\right] - \mathbb{E}_{f_k}\left[G\left[
      \left\{k(t)\right\}_{t\in[\tau-1]}
    \right]\right],
  \end{equation}
  where $\tilde{k}(t) = k(t)$ when $t\in[\tau-1]$ and $\tilde{k}(t) = k$ when $t = \tau$.
\end{definition}
\begin{definition}[UCB of EI]\label{def1}
A representation of a UCB of $EI[k,\tau;G]$ is denoted as $z(k,\mathcal{R}(\tau-1);G)$,
where $\mathcal{R}(\tau)\coloneqq\left\{\{k(t),r_{k(t)}(t)\}\right\}_{t\in\tau}$
is the set of the pairs of the selected arm ids and rewards previously played and obtained.
\end{definition}
\begin{definition}[Sub-Gaussian]\label{sub-gaussian}
  A distribution $f(r)$ is called a sub-Gaussian distribution
  when $\exists m \in \mathbb{R}$ and $\exists s \in \mathbb{R}^{+}$, such that
  \begin{equation}
    \modifyA{\mathbb{P}_f[|r|<u]} \leq U(u;m,s^2),
  \end{equation}
  where
  \begin{equation}
    U(r;m,s^2) \coloneqq 2 \exp\left[-\frac{(r - m)^{2}}{2s^{2}}\right].
  \end{equation}
  The $m$ and $s^{2}$ are called the mean and variance proxies, respectively.
\end{definition}
\begin{definition}
\begin{equation}
  I(r;m,s^2) \coloneqq \int_r^\infty U(u;m,s^2)du = \sqrt{2\pi s^2}\erfc\left[\frac{r-m}{\sqrt{2 s^2}}\right],
\end{equation}
where $\erfc(x)$ denotes the complementary error function.
\end{definition}
\begin{definition}[Pseudo-Upper Bound]
  The symbol $\lessapprox$ means that the right value
  is a pseudo-value of the upper bound of the left value.
\end{definition}
\begin{definition}[Sub-exponential]\label{sub-exponential}
  Distribution $g(x)$ is called a sub-exponential when $\exists b \geq 0$, such that
  \begin{equation}
    \mathbb{P}_{g(x)} \{ x \geq u \} \leq 2 \exp{ \left( -\frac{u}{b} \right) }.
  \end{equation}
\end{definition}


\section{Our Concept}\label{sec4}
Our main claim of this article is the effectiveness of Algotihm \ref{maxSearch}
which uses a UCB of EI of the single best reward as the selection index.
We first show the reasonability of the use of a UCB of EI for the bandit algorithm
by taking the conventional UCB \citep{auer2002finite,bubeck2013bandits} as an example.
This example is intuitively clear although it is not rigorous.
We provide a more theoretical discussion in Section \ref{sec7}.
We also show in Lemmas \ref{EI1} and \ref{EI2}
that the EI of the single best reward can be calculated from a survival function of $f_k(r)$.
The substantial value is estimated in Theorem \ref{pseudoUCB} in the next section.

In the conventional bandit problem,
the EI becomes the expected reward as shown in the following lemma.
\begin{lemma}[EI of conventional bandit problem]
  In the conventional bandit problem,
  \begin{equation}
    EI\left[k,\tau;G^\text{sum}\right] = \mathbb{E}_{f_k}[r].
  \end{equation}
\end{lemma}
Based on this proposition,
we can consider that the conventional UCB algorithm \citep{auer2002finite} uses
a UCB of $EI[k,\tau;G^\text{sum}]$ as a selection index.
This relation of the conventional UCB and EI implies
that the same approach is valid in the MKB problem.
Namely, Algorithm \ref{maxSearch} using 
$z(k,\mathcal{R}(\tau-1);G^\text{sum})$ in Definition \ref{def1}
as a selection index can be generated conceptually.
Because Definition \ref{def1} does not state
the substantial form of $z(k,\mathcal{R}(\tau-1);G^\text{max})$,
we should estimate it to use Algorithm \ref{maxSearch}.
The following proposition and theorem are footholds for the estimation.
\begin{lemma}[EI of MKB problem]\label{EI1}
  In the MKB problem, let $r^\text{max}\coloneqq\max_{t\in[\tau-1]}r_{k(t)}(t)$ be given.
  Then,
  \begin{equation}
    EI[k,\tau;G^\text{max}]
    = \mathbb{E}_{f_k} \left[ \max\{r_{k(t)}(t),r^\text{max}\} \right] - r^\text{max}.
  \end{equation}
\end{lemma}
\begin{lemma}[EI and Survival Function]\label{EI2}
  Let $r$ be an independent identical distributed (i.i.d.) random variable
  following $f(r)$ and $r_0$ be given.
  Then,
  \begin{equation}
    \mathbb{E}_f\left[\max\left\{r,r_0\right\}\right]-r_0=\int_{r_0}^\infty S(u)du,
  \end{equation}
  where
  \begin{equation}
    S(r) \coloneqq \int_r^\infty f(u) du
  \end{equation}
  is the survival function of $f(r)$.
\end{lemma}
\begin{proof}
  \begin{equation}
  \begin{split}
  \mathbb{E}_f\left[\max\left\{r,r_0\right\}\right]
  &= \int_{-\infty}^{r_0}r_0f(r)dr+\int_{r_0}^\infty r f(r)dr \\
  &= \int_{-\infty}^{r_0}r_0f(r)dr+\int_{r_0}^\infty \int_0^r du f(r)dr \\
  &= \int_{-\infty}^{r_0}r_0f(r)dr+ \int_0^{r_0} du \int_{r_0}^\infty f(r) dr 
  + \int_{r_0}^\infty du \int_{u}^\infty f(r) dr \\
  &= r_0 + \int_{r_0}^\infty S(u) du.
  \end{split}
  \end{equation}
  Here, in switching the order of the integration, we used
  \begin{equation}
  r_0\leq r \cap 0\leq u\leq r \Leftrightarrow
  (0\leq u\leq r_0 \cap r_0\leq r) \cup (r_0\leq u \cap u\leq r).
  \end{equation}
\end{proof}
Lemma \ref{EI1} assumes $r^\text{max}$ given.
It is no problem in the implementation because $r^\text{max}$ can be recorded a $O(1)$ memory.
Lemma \ref{EI2} says that the EI in Lemma \ref{EI1} can be calculated
from the survival function of the reward distribution.
Because of this, the remained work to obtain the selection index is 
the estimation of a substantial form of a UCB of $\int_{r^\text{max}}^\infty S_k(r)dr$
with some assumption for the reward distribution.

\renewcommand{\algorithmicrequire}{\textbf{Input:}}
\renewcommand{\algorithmicensure}{\textbf{Output:}}
\begin{algorithm}[h]
  \caption{MaxSearch}\label{maxSearch}
  \begin{algorithmic}[1]
    \REQUIRE
    number of arms $K$,
    current time $\tau$,
    and previous records $\mathcal{R}(\tau-1)$.
    \ENSURE selected arm index $\hat{k}$.
    \FOR{\textbf{each} $k\in[K]$}
    \STATE calculate $z_k = z(k,\mathcal{R}(\tau-1);G^\text{max})$
    \ENDFOR
    \STATE $\hat{k} \leftarrow \argmax_{k\in[K]}z_k$
    \RETURN $\hat{k}$
  \end{algorithmic}
\end{algorithm}

\section{Estimation for Selection Index}\label{sec5}
Employing the sub-gaussian assumption in the reward distributions,
\begin{equation}
\int_{r^\text{max}}^\infty S_k(r)dr \leq I(r^\text{max};m_k,s^2_k). 
\end{equation}
Therefore, we can use a UCB of $I(r;m,s^2)$ as the selection index
$z(k,\mathcal{R}(\tau-1);G^\text{sum})$.
We could only estimate a pseudo-value of this UCB as follows:
\begin{theorem}[Pseudo-UCB of $I(r;m,s^2)$]\label{pseudoUCB}
  Let $f(r)$ be a sub-Gaussian distribution with mean proxy $m$ and variance proxy $s^2$.
  Let $S(r)$ be a survival function of $f(r)$.
  Let $r_{i}, i\in[n]$ be i.i.d. sub-Gaussian random variables following $f(r)$.
  Let $r^\text{max}$ be $\max_{i\in[n]}r_i$.
  Then, 
  \begin{equation}
    I(r^\text{max};m,s^2) \lessapprox I(r^\text{max};\tilde{m},\tilde{s}^2)
  \end{equation}
  with the confidential level $0<1-\alpha<1-\exp[-(\sqrt{2}-\sqrt{2-\ln{2}})^2n]$,
  or almost equivalently $n>-13.613\ln\alpha>0$.
  where
  \begin{equation}
    \tilde{m} \coloneqq \frac{1}{n}\sum_{i=1}^{n}r_{i},
  \end{equation}
  \begin{equation}
    \tilde{s}^{2} \coloneqq \frac{1}{2[\ln 2 - \gamma(\alpha,n)]}
    \left(\frac{1}{n}\sum_{i=1}^{n}r_{i}^2 - \tilde{m}^2\right),
  \end{equation}
  and
  \begin{equation}
    \gamma(\alpha,n) \coloneqq \frac{\ln \alpha}{n} + 2\sqrt{\frac{-2\ln\alpha}{n}}.
  \end{equation}
\end{theorem}
This theorem is weak due to state only a pseudo-value.
However, Algorithm \ref{maxSearch} with the selection index
calculated from Algorithm \ref{selectionIndex},
which is based on Theorem \ref{pseudoUCB},
demonstrated a good performance shown in Section \ref{sec6}.
\begin{algorithm}[h]
  \caption{PseudoUCB}\label{selectionIndex}
  \begin{algorithmic}[1]
    \REQUIRE
    total number of selections previously performed $\nu$,
    number of times the target arm is selected $n$,
    sum of the rewards obtained from the target arm $R$,
    sum of the square rewards obtained from the target arm $Q$,
    the maximum reward obtained until the current time $r^{\max}$,
    and the hyperparameter $c$.
    \ENSURE selection index $z$.
    \IF{$n = 0$ or $\nu < 2$}
    \STATE $z \leftarrow \infty$
    \ELSE
    \STATE $\beta \leftarrow c\sqrt{(\ln\nu)/n}$
    \STATE $\gamma \leftarrow -\beta^2 + 2\sqrt{2}\beta$
    \IF{$\gamma>\ln 2$}
    \STATE $z \leftarrow \infty$
    \ELSE
    \STATE $\tilde{m} \leftarrow R/n$ 
    \STATE $\tilde{s}^2
    \leftarrow (Q/n - \tilde{m}^2)/[2(\ln 2 - \gamma)]$ 
    \STATE $z \leftarrow \sqrt{2\pi\tilde{s}^2}
      \erfc \left[(r^{\max} - \tilde{m})/\sqrt{2\tilde{s}^{2}}\right]$
    \ENDIF
    \ENDIF
    \RETURN $z$
  \end{algorithmic}
\end{algorithm}\\
The weakness of Theorem \ref{pseudoUCB} comes from no upper bound of $s^2$
in the sub-gaussian assumption, indicated in Proposition \ref{noUpperBound}.
We alternatively use a lower bound of $s^2$
in this proposition to derive Theorem \ref{pseudoUCB}.
The term "pseudo" in this theorem represents the theoretical inauthenticity of this treatment.
The derivation of Theorem \ref{pseudoUCB} is based on
Bernstein's inequality in Proposition \ref{Bernstein}
because the lower bound of $s^2$ contains the expected square value of the reward.
See the following subsections for details.

\subsection{Sub-Gaussian assumption}
We aim to apply the MKB algorithm into materials discovery.
Materials properties usually are approximately normally distirbuted.
However, they are sometimes bounded and may have jumps by material group.
Fortunately, many properties are not heavily-tailed,
so the sub-gaussian assumption is an innocuous assumption for our purpose. 

Using the sub-gaussian assumption,
a UCB of $I(r;m,s^2)$ is given by the following conceptual proposition.
\begin{lemma}[UCB of $I(r;m,s^2)$]\label{UCB}
  Consider $n$ samples $\{r_i\}, i\in [n]$
  taken from the sub-Gaussian distribution $f(r)$
  with mean proxy $m$ and variance proxy $s^2$.
  Let $\tilde{m}$ and $\tilde{s}^2$ be UCBs of the mean and variance proxies
  with the confidence level $0<1-\alpha<1$, respectively.
  Then,
  \begin{equation}
    I(r;m,s^2) \leq I(r;\tilde{m},\tilde{s}^2)
  \end{equation}
  with the confidence level $1-\alpha$, when $r\geq m$.
\end{lemma}
\begin{proof}
With the confidence level $1-\alpha$, $m\leq\tilde{m}$ and $s^2\leq\tilde{s}^2$.
$I(m;m,s^2) \leq I(m;\tilde{m},\tilde{s}^2)$.
$I(r;m,s^2)$ is monotonically decreasing for $r$ and increasing for $m$ and $s^2$.
Then, $I(r;m,s^2) \leq I(r;\tilde{m},\tilde{s}^2)$ when $r\geq m$.
\end{proof}
In this lemma, UCBs denoted as $\tilde{m}$ and $\tilde{s}^2$ are virtual.
We represent it using the term "conceptual".
Badly, as the following lemma indicates,
even the upper bound of $s^2$ cannot be determined under the sub-gaussian assumption.
\begin{lemma}[Bounds on $s^2$]\label{noUpperBound}
Let $f(r)$ is a sub-gaussian with mean proxy $m$ and variance proxy $s^2$.
Then,
\begin{equation}
2s^2\geq\frac{\mathbb{E}_{f}[(r - m)^{2}]}{\ln 2}.
\end{equation}
There are no upper bounds on $s^2$.
\end{lemma}
\begin{proof}
  Lower bound:
  As an equivalent condition to the sub-gaussian on Definition \ref{sub-gaussian},
  the following Orlicz condition is established \citep{vershynin2018high}.
  \begin{equation}
    \mathbb{E}_{f}\left[\exp\left[\frac{(r - m)^{2}}{2s^{2}}\right]\right] \leq 2.
  \end{equation}
  Applying Jensen's inequality to this condition,
  we obtain $\exp\left[\mathbb{E}_{f}[(r - m)^{2}]/(2s^{2})\right]\leq 2$.
  Then, $2s^2 \geq \mathbb{E}_{f}[(r - m)^{2}]/\ln 2$.

  No upper bound: 
  From the sub-gaussian defintion, $f(r) \leq U(r;m,s^2)$.
  Then, $f(r)\leq U(r;m,\hat{s}^2)$, where $\hat{s}^2>s^2$.
  Therefore, $f(r)$ can be considered as a sub-gaussian with variance proxy $\hat{s}^2>s^2$.
  It means no upper bound of $s^2$.
\end{proof}
Because of this proposition, we gave up deriving a rigor UCB of $s^2$.
Alternatively, we decided to use a UCB of the lower bound
$\mathbb{E}_{f}[(r - m)^{2}]/\ln 2$ as a pseudo-UCB.

\subsection{Derivation for Theorem \ref{pseudoUCB}}
To estimate a UCB of $\mathbb{E}_{f}[(r-m)^2]/\ln 2$,
we first indicate the sub-exponential property of $(r-m)^2$ in Lemma \ref{squareReward}.
An UCB of the expected value of the sub-exponential distribution is known to be estimated
from Bernstein's inequality \citep{vershynin2018high}.
Therefore, we estimate a UCB of $\mathbb{E}_{f}[(r-m)^2]/\ln 2$
using a variant of Bernstein's inequality in Proposition \ref{Bernstein}.
This result gives a pseudo-UCB of $s^2$ in Lemma \ref{pseudoUCB2}.
Then, applying Lemma \ref{pseudoUCB2} into Lemma \ref{UCB},
we obtain Theorem \ref{pseudoUCB}.

As is shown in the following lemma, $(r-m)^2$ becomes sub-exponential
under the sub-gaussian assumption.
\begin{lemma}[Sub-Exponential Property of Square Reward]\label{squareReward}
  Let $r$ be an i.i.d. random variable following sub-Gaussian $f(r)$ with variance proxy $s^2$.
  Then, $\forall m\in\mathbb{R}$, $x = (r - m)^{2}$
  follows a sub-exponential with parameter $b = 2s^{2}$.
\end{lemma}
\begin{proof}
  Let $r$ be an i.i.d. random variable following a sub-Gaussian $f(r)$
  with mean proxy $m$ and variance proxy $s^2$.
  Let $g(x)$ be the distribution of $x = (r - m)^{2}$.
  Then, $\forall u \geq 0$,
  \begin{equation}
    \begin{split}
      \mathbb{P}_{g}\{ x \geq u \}
      &= \mathbb{P}_{f}\{ (r - m)^{2} \geq u \}
      = \mathbb{P}_{f}\{ | r - m | \geq \sqrt{u} \}\\
      &\leq 2 \exp{ \left( -\frac{u}{2s^{2}} \right) }
      = 2 \exp{ \left( -\frac{u}{b} \right) },
    \end{split}
  \end{equation}
  where $b = 2s^{2}$.
  We used the definition of sub-exponential property for the last inequality.
\end{proof}
Using the sub-exponential property of $(r-m)^2$,
a UCB of $\mathbb{E}_{f}[(r - m)^{2}]$
can be estimated from a variant of Bernstein's inequality.
\begin{proposition}[Bernstein's inequality]\label{Bernstein}
  Let $g(x)$ be a sub-exponential distribution with parameter $b$.
  Let $x_{1}, x_{2}, \ldots, x_{n} > 0$
  be i.i.d. sub-exponential random variables following $g(x)$.
  Then,
  \begin{equation}
    \mathbb{P}_g \left\{ \frac{1}{n} \sum_{i=1}^{n} x_{i} - \mathbb{E}_g\left[ x \right]
    \geq u \right\} \leq \exp{( 2\sqrt{2} n w_{+} - n w_{+}^{2} - 2n)},
  \end{equation}
  and
  \begin{equation}
    \mathbb{P}_g \left\{
    \mathbb{E}_g\left[ x \right] - \frac{1}{n} \sum_{i=1}^{n} x_{i} \geq u 
    \right\}
    \leq \left\{ \begin{array}{ll}
      \exp{ ( 2\sqrt{2} n w_{-} - n w_{-}^{2} - 2n) } & 0 \leq u < 3b/2\\
      \exp{ \left[- \left( ub^{-1} + 1 \right) n \right] } & 3b/2 \leq u,
    \end{array} \right.
  \end{equation}
  where $u \geq 0$ and $w_{\pm} \coloneqq \sqrt{ \pm ub^{-1} + 2 }$.
\end{proposition}
The proof is presented in Appendix A.
From Lemma \ref{squareReward} and Proposition \ref{Bernstein},
a UCB of $\mathbb{E}_f[(r - m)^2]$ is given as follows:
\begin{corollary}[Confidence Bounds of $\mathbb{E}_f[(r-m)^2\rbrack$]\label{UCB2}
Consider $n$ samples $\{r_i\}_{i\in [n]}$
taken from the sub-Gaussian distribution $f(r)$
with mean proxy $m$ and variance proxy $s^2$.
Then,
\begin{equation}
  \frac{1}{n} \sum_{i=1}^{n} (r_{i}-m)^2 - \gamma_+b
  \leq \mathbb{E}_f\left[ (r-m)^2 \right]
  \leq \frac{1}{n} \sum_{i=1}^{n} (r_{i}-m)^2 + \gamma_-b,
\end{equation}
when $\gamma_+ \coloneqq \beta^{2} + 2\sqrt{2}\beta$ and
\begin{equation*}
  \gamma_{-} \coloneqq \left\{ \begin{array}{ll}
    -\beta^2 + 2\sqrt{2}\beta & 0 < \beta < 1/\sqrt{2}\\
    1 + \beta^{2} & 1/\sqrt{2} \leq \beta,
  \end{array} \right.
\end{equation*}
with a confidence level of $0 < 1 - \alpha < 1$,
where $\beta \coloneqq \sqrt{ -\ln \alpha/n }$ and $b\coloneqq 2s^2$.
\end{corollary}
Using this corollary, we obtain a pseudo-UCB of $s^2$ as follows:
\begin{lemma}[Pseudo-UCB of $s^2$]\label{pseudoUCB2}
Consider $n$ samples $\{r_i\}_{i\in [n]}$
taken from the sub-Gaussian distribution $f(r)$
with mean proxy $m$ and variance proxy $s^2$.
Then,
\begin{equation}\label{stilde}
  s^2 \lessapprox \frac{1}{2n(\ln 2-\gamma)} \sum_{i=1}^{n} (r_{i} - m)^{2}
\end{equation}
with the confidence level $0<1-\alpha<1-\exp[-(\sqrt{2}-\sqrt{2-\ln{2}})^2n]$,
or almost equivalently $n>-13.613\ln\alpha>0$.
$\gamma \coloneqq -\beta^2 + 2 \sqrt{2} \beta$, where $\beta \coloneqq \sqrt{-\ln \alpha/n}$.
\end{lemma}
\begin{proof}
Substituting $2s^2\ln 2=\mathbb{E}_f[(r - m)^2]$ into Corollary \ref{UCB2}
\begin{equation}
  2s^2\ln 2 \leq \frac{1}{n} \sum_{i=1}^n (r_i - m)^2 + 2s^2 \gamma_-.
\end{equation}
Then, when $0<\gamma_-<\ln 2<1/\sqrt{2}$,
\begin{equation}
  s^2 \leq \frac{1}{2n(\ln 2-\gamma_-)} \sum_{i=1}^n(r_i-m)^2.
\end{equation}
From the bounds of $\gamma_-$, $0<\beta<\sqrt{2}-\sqrt{2-\ln 2}$,
which is equivalent to the bounds of $\alpha$ in this lemma.
\end{proof}

Lemmas \ref{UCB} and \ref{pseudoUCB2} contain unknown $\tilde{m}$ and $m$.
We simply select $m = \tilde{m} = (\sum_{i=1}^{n}r_{i})/n$
because we only estimated the pseudo-value of $\tilde{s}^2$.
Then, Theorem \ref{pseudoUCB} is obtained from these lemmas.
Using the sample means for $m$ and $\tilde{m}$ is
also justified in terms of the order of convergence.
The confidence bounds of the sample mean of the reward
converge to the expected mean at $O(n^{-1/2})$ \citep{auer2002finite}.
Then, we expect that $m$ and $\tilde{m}$ also converge in the same order.
This order is faster than $O(n^{-1/4})$,
which is the order of convergence of $\tilde{s}$ in Lemma \ref{pseudoUCB2}.
In this case, $I(r;\tilde{m},\tilde{s}^2)$ in Lemma \ref{UCB} converges to
$\lim_{n\to\infty}I(r;\tilde{m},\tilde{s}^2)$ at the same order of $\tilde{s}$.
Because of this, it is sufficient to correctly evaluate only $\tilde{s}$.

\subsection{Derivation for Algorithm \ref{selectionIndex}}
Setting $\alpha = \nu^{-c^{2}}$, where $\nu$ is the current number of selections,
we can implement UCB in Theorem \ref{pseudoUCB} as Algorithm \ref{selectionIndex}.
We use $\nu$ instead of $\tau-1$ because of the generality in MCTS.
In this algorithm, $c$ is a hyperparameter
that controls the balance between exploration and exploitation.
We recommend $c = 1/\sqrt{13.613}$
to satisfy the condition, $\gamma < \ln 2$, when $n>\ln\nu$.
To explore the search space more randomly, a larger $c$ should be used.
In the implementation, when the same $z$ value was obtained from other arms,
one of the arms was selected randomly.

The inequality, $n>\ln\nu$, indicates that
our algorithm allocates at least $\ln\nu$ trials for the non-optimal arms.
This allocation order is consistent with the optimal order
for the conventional bandit problem \citep{auer2002finite},
whereas it is inconsistent with the double exponential order,
which is optimal in the MKB with the Gumbel-type reward distribution \citep{cicirello2005max}.
Considering the uncertainty of the MKB problem discussed in Section \ref{sec7},
the optimal allocation order will be explored in future work.
}


\section{Experiments and Results}\label{sec6}
We conduct two types of numerical experiments to compare our algorithm with other algorithms.
One is the synthetic bandit problems with the Gaussian reward distributions,
and the other is SMILES optimization using MCTS
\citep{yang2017chemts,kajita2020autonomous,kikkawa2020self}
as the demonstrations for materials discovery.
We employed a single set of recommended or reasonable hyperparameters
for all the experiments because the tuning of hyperparameters
for the actual applications in materials discovery is extremely expensive.
We set $T=10,000$ considering the realistic applications \citep{kajita2020autonomous,kikkawa2020self}
unless the observed maximum reward clearly does not converge.
We present the details of other algorithms in Appendix \ref{app2}.

\subsection{Synthetic problems for bandits}
The synthetic problems explored in the experiments include the following:
\begin{description}
\item[``easy'' problem]\mbox{}\\
  This problem consists of three arms with the Gaussian parameters
  $(\mu_{1}, \sigma_{1}) = (1,1)$, $(\mu_{2}, \sigma_{2})$ $= (0, 2)$, and $(\mu_{3}, \sigma_{3}) = (-1, 3)$.
  Arm 3 is optimal for the MKB problem because of its large variance.
  However, in the conventional bandit approaches,
  arm 1 is preferred because of its high expectation reward.
\item[``difficult'' problem]\mbox{}\\
  This problem consists of three arms with $(\mu_{1}, \sigma_{1}) = (-0.2, 1.1)$,
  $(\mu_{2}, \sigma_{2}) = (0, 1)$, and $(\mu_{3}, \sigma_{3}) = (-0.8, 1.2)$.
  In this problem, the optimal arm in the MKB problem switches depending on the total number of trials.
  Arm 1 is optimal $10^2 \ll T \ll 10^9$ because $\mu_1+2\sigma_1=\mu_2+2\sigma_2$
  and $\mu_1+6\sigma_1=\mu_3+6\sigma_3$.
  The algorithms for determining the arm with the maximum expectation reward will select arm 2.
  An algorithm with a strong tendency to choose arms
  with high variances will have a higher preference toward arm 3 than arm 1.
  It is a challenge for the MKB algorithm to select arm 1 correctly.
\item[``unfavorable'' problem]\mbox{}\\
  This problem comprises three arms with the same variance;
  the Gaussian parameters of each arm were set to
  $(\mu_{1}, \sigma_{1}) = (1, 1)$, $(\mu_{2}, \sigma_{2}) = (0, 1)$,
  and $(\mu_{3}, \sigma_{3}) = (-1, 1)$, respectively.
  In this setting, arm 1 is optimal.
  the conventional UCB will select the optimal arm correctly
  because this arm has the highest mean reward.
  The MKB algorithms will lose the conventional UCB
  because these algorithms incur costs for estimating the variance of each arm.
\end{description}

The transition plots of the observed maximum and the ratio of the optimal arm selection
averaged over 100 independent runs are shown in Figure \ref{fig1}.
The plots of the observed maximum can directly evaluate the performance of MKB;
however it is susceptible to data variability.
The ratio of the optimal arm selection can help in that case.

In the result of the ``easy'' problem, the MKB algorithms
exhibit higher observed maximum reward than the random search on average.
Although the obtained maximum rewards are similar among these MKB algorithms,
the ratios of the optimal arm selected clearly show that our algorithm identifies the best arm first.
As expected, the conventional UCB mainly selected the non-optimal arm.
The {\it spUCB} and {\it UCBE} also afforded worse results than those of the random search.

In the ``difficult'' problem, the selection ratios show that
the MKB algorithms selected the optimal arm more frequently than the random search,
although slight differences were observed in the observed maximum reward.
In particular, our algorithm was the most efficient in selecting the optimal arm.
The performances of the random search and {\it UCBE} were almost the same,
and {\it spUCB} and the conventional UCB exhibited the worse performances.

In the ``unfavorable'' case, the conventional UCB worked the best
from the viewpoint of the selection ratio.
The performance of {\it spUCB} is similar to that of the conventional UCB.
Our algorithm also exhibited good performance,
although the ratios were slightly lower than those of the conventional UCB.
The results of {\it ThresholdAscent} and {\it RobustUCBMax}
were better than those of {\it UCBE}.
The random search afforded the worst result.

\begin{figure}[htbp]
  \centering
  \begin{tabular}{c}
    \includegraphics[width=15cm]{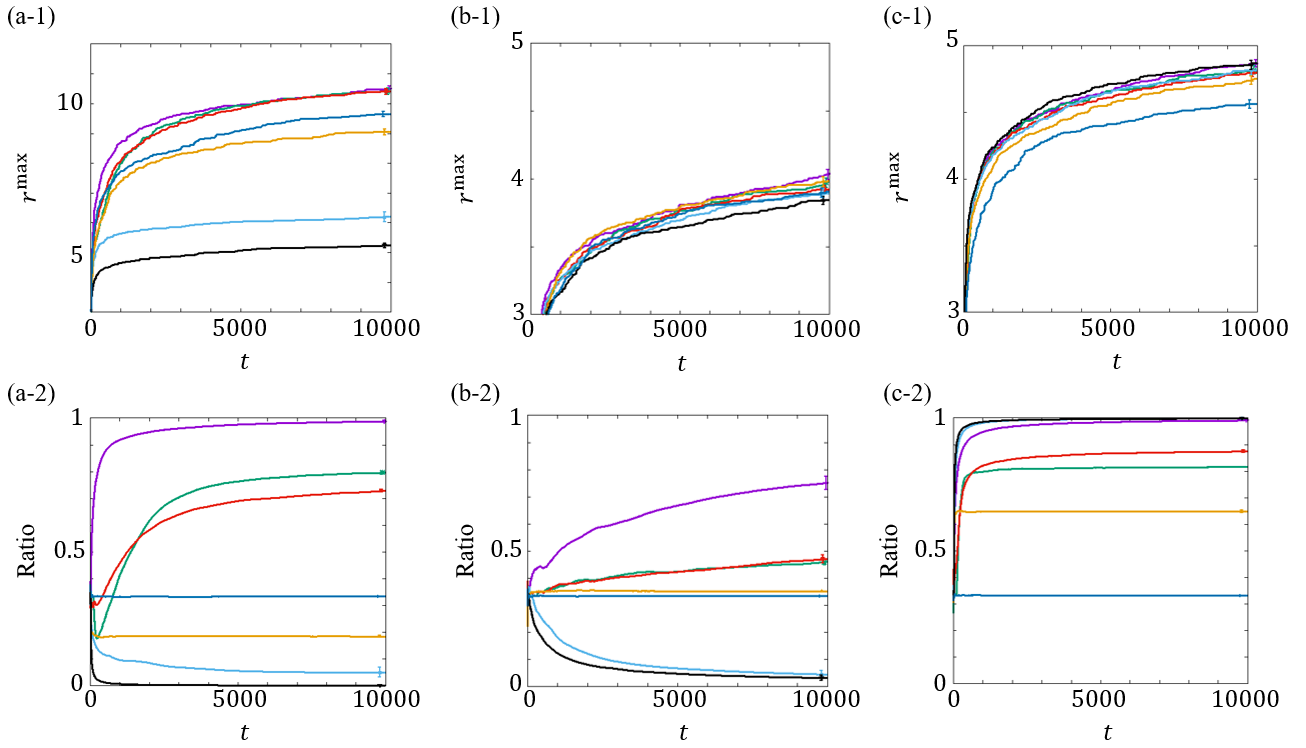}
  \end{tabular}
  \caption{
    Transition plots of ({\it x}-1) the observed maximum
    and ({\it x}-2) the ratio of the optimal arm selection.
    (a) ``Easy'' problem, (b) ``difficult'' problem, and (c) ``unfavorable'' problem.
    The colors represent the following: purple, this work; green, {\it ThresholdAscent};
    red, {\it RobustUCBMax}; sky blue, {\it spUCB}; orange, {\it UCBE};
    black, conventional UCB; and blue, random search.
    The error bars indicate the standard errors of 100 independent runs.
    If the differences between the two methods are more than two times the standard errors, 
    there will be a significant difference between these methods with a 5 \% significance level.
  }\label{fig1}
\end{figure}

\subsection{Molecular discovery using tree search}
As a demonstration of the molecular discovery problem,
we attempted to optimize the molecular structure $M$
which maximized either of the properties
defined by the following empirical equations \citep{joback}:
\begin{equation*}
  T_\text{b}(M) \text{[K]} = 198.2 + \sum_{i \in \text{frag}(M)} T_{\text{b},i},
\end{equation*}
\begin{equation*}
  P_\text{c}(M) \text{[bar]} = \left[ 0.113 + 0.0032N_\text{a}(M) + \sum_{i \in \text{frag}(M)} P_{\text{c},i} \right]^{-2},
\end{equation*}
\begin{equation*}
  \eta_{300 \text{K}}(M) \text{[Pa$\cdot$s]}
  = M_\text{w}(M) \exp{\left[ \frac{\sum_{i \in \text{frag}(M)}\eta_{\text{a},i} - 597.82}{300} + \sum_{i \in \text{frag}(M)}\eta_{\text{b},i} - 11.202 \right] },
\end{equation*}
where $ T_\text{b} $, $ P_\text{c} $, and $ \eta_{300\text{K}} $
are the boiling temperature, critical pressure,
and liquid dynamic viscosity at $300$ K of molecule $M$, respectively;
$\text{frag}(M)$ was a set of atomic fragments of $M$, determined by Joback and Raid.
The fragments simply determined for each atom type,
such as carbon in methyl group, halogens, and ether oxygen in a ring group, etc.
\modifyA{The functions} $N_\text{a}(M)$ and $M_\text{w}(M)$ were the number of atoms in $M$ and molecular weight of $M$, respectively.
The empirical parameters, $T_{\text{b},i}$, $P_{\text{c},i}$, $\eta_{\text{a},i}$,
and $\eta_{\text{b},i}$, were optimized to reproduce the experimental properties.
The properties, $T_\text{b}$, $P_\text{c}$, and $\eta_{300 \text{K}}$,
depended on the molecular structure through these parameters.
In addition to those three properties,
the topological polar surface area $\text{TPSA}(M)$ [$\text{\AA}^2$]
($\text{\AA} = 0.1 \text{nm}$) \citep{ertl} was maximized.
Using these empirical formulas, we can verify
the performance of the search algorithms in a short time.

During the search process, the candidate molecular structures were generated
using the following context-free grammar \citep{hopcroft}
of the SMILES strings \citep{weininger}.
Using the context-free grammar,
we could create a simple maze game \citep{kikkawa2020self} systematically.
Here, we applied the following rules:
\begin{equation*}
  \begin{split}
    S &\rightarrow \text{C($X$)($Y$)($Y$)($Y$)}, \text{C(=O)($Y$)($Y$)},
    \text{C($Y$)C($Y$)(=C($Y$)C($Y$))}, {\rm or} \; \text{C(=O)(O($Y$))($Y$)},\\
    X &\rightarrow \text{[H]}, \text{F}, \text{Cl}, \text{Br}, \text{C($X$)($Y$)($Y$)},
    \text{O($Y$)}, \text{N($Y$)($Y$)}, \text{C(=O)($Y$)},\\
    \phantom{X} & \phantom{\rightarrow, ,} 
    \text{C($Y$)(=C($Y$)($Y$))}, {\rm or} \; \text{C(=O)(O($Y$))},\\
    Y &\rightarrow \text{[H]}, \text{F}, \text{Cl}, \text{Br}, \text{C($X$)($Y$)($Y$)},
    \text{C(=O)($Y$)}, \text{C($Y$)(=C($Y$)($Y$))}, {\rm or} \; \text{C(=O)(O($Y$))},
  \end{split}
\end{equation*}
where $S$, $X$, and $Y$ denote the non-terminal variables,
and the upright characters denote the terminals.
The start variable is set to $S$, and a string-generation process
is completed when the string no longer has variables.
The following additional rule was applied
when the number of alphabets was greater than 40:
\begin{equation*}
  X \; {\rm or} \; Y \rightarrow \text{[H]}.
\end{equation*}
This rule guarantees the termination of the generation process
within the moderate molecular size.
This limit is approximately 500 g/mol in molecular weight,
and most of the known molecules in the database%
\footnote{\url{https://www.rsc.org/Merck-Index/}} are within the limit.
We employed hydrogen as the termination atom, which is commonly used in organic chemistry.
The alphabets include the explicit ``H'', and exclude the parenthesis and equal symbols.
The string ``Br'' and ``Cl'' are considered as two alphabets.
The search space of this molecular generator
contains significantly more than $6.248\times10^{13}$ molecular species,
which is the number of isomers in $\text{C}_{40}\text{H}_{82}$ \citep{yeh1995isomer}.
We did not consider the synthesizability and the target scope of generated molecules;
however, it can be considered by modifying the grammar in practical use.

The context-free language can be projected to a tree graph (Fig.\ref{fig4}).
Therefore, the molecular generator can be easily implemented with an MCTS algorithm,
as shown in Algorithm \ref{mcts}.
The node selection in each layer continues until a complete molecular string is created.
Subsequently, the chemical property evaluation is performed,
after which the property value is used as the reward.
The reward value is recorded in each node passed in the creation,
and it is used to calculate the selection indices in the next creation.
The complete SMILES strings assigned on the different leaves
are treated as the different molecules in this search algorithm
even if these molecules have the same molecular symmetries.
\begin{figure}[htbp]
  \centering
  \begin{tabular}{c}
    \includegraphics[width=9cm]{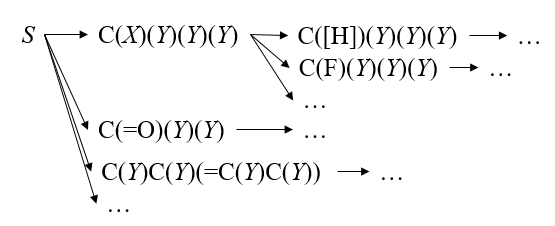}
  \end{tabular}
  \caption{Tree image of SMILES generation.}\label{fig4}
\end{figure}

\begin{algorithm}[h]
  \caption{MCTS}\label{mcts}
  \begin{algorithmic}[1]
    \REQUIRE number of trials $ T $.
    \STATE $ \tau = 0 $
    \WHILE{$ \tau < T $}
    \STATE $\tau \leftarrow \tau + 1$
    \STATE $v \leftarrow root()$
    \COMMENT {the root node of the search tree.}
    \STATE $L \leftarrow \{v\}$
    \WHILE{$v$ is not a leaf node}
    \STATE $k \leftarrow policy(records)$\\
    \COMMENT {$policy$ : Algorithm \ref{maxSearch} or the algorithms in Appendix B.\\
      $records$ : statistic data such as $K$, $n_k$, $R_k$, $R_k^2$,
      and $r^{\max}$ in Algorithm \ref{maxSearch}.}
    \STATE $L \leftarrow L \cup k$
    \STATE $v \leftarrow child(k, v)$
    \COMMENT {the $k$-th child of the node $v$.}
    \ENDWHILE
    \STATE $r \leftarrow reward(L)$
    \COMMENT {the reward of the selected path $L$.}
    \STATE $records.add(L, r)$
    \COMMENT {record the path $L$ and the reward $r$.}
    \ENDWHILE
  \end{algorithmic}
\end{algorithm}

The properties, $T_\text{b}$, $P_\text{c}$, and $\eta_{300\text{K}}$
were calculated using the python thermo.joback module \citep{bell},
and TPSA was calculated using the RDKit library \citep{landrum}.
When using $\eta_{300\text{K}}$ as the reward,
the rules containing one of F, N, and =C were excluded
because their empirical parameters were not available.
Additionally, we note that all of the generated SMILES were valid in the network test of RDKit.

Using the transition plots of the observed maximum,
we compared {\it MaxSearch} and other algorithms in Figure \ref{fig3}.
The plots of $T_\text{b}$, $P_\text{c}$, and TPSA were obtained
by averaging over each 100 independent search runs.
The plots of $\eta_{300\text{K}}$ were medians of the 100 runs with quantile error bars.
Because of the large reward dispersion and skewness of $\eta_{300\text{K}}$,
this treatment was required for the graph readability.
\begin{figure}[htbp]
  \centering
  \begin{tabular}{c}
    \includegraphics[width=12cm]{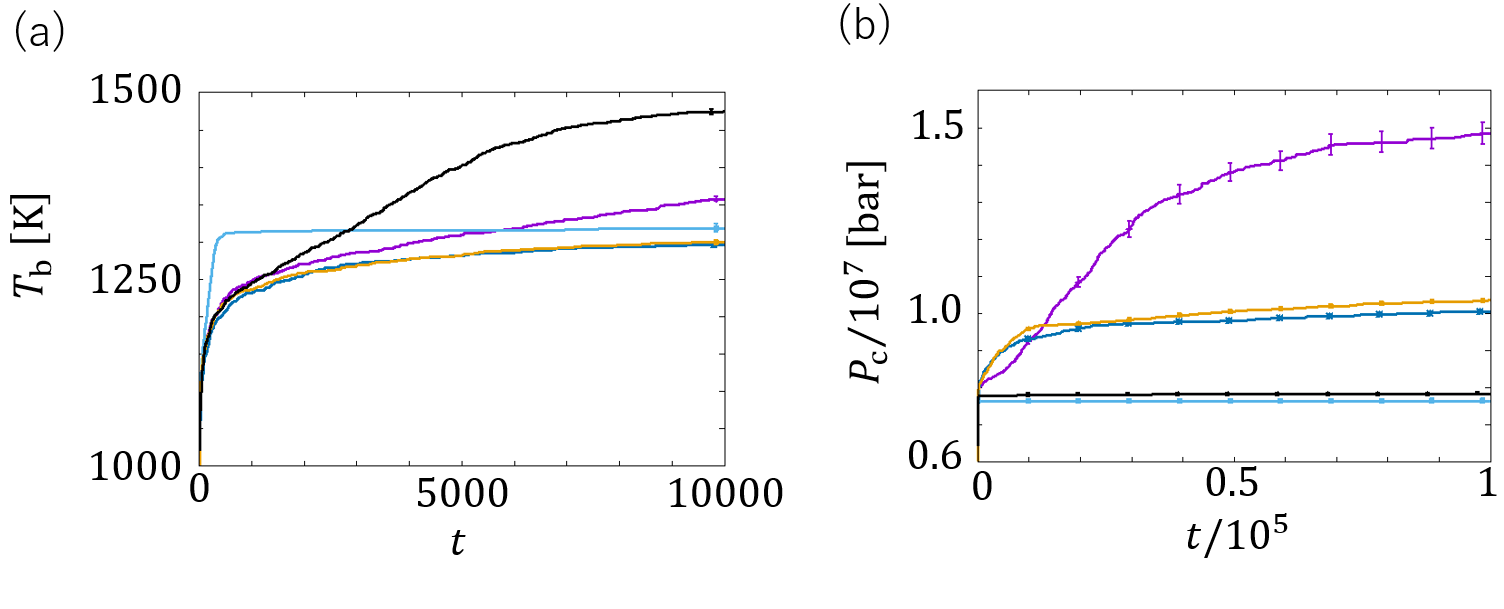}\\
    \includegraphics[width=12cm]{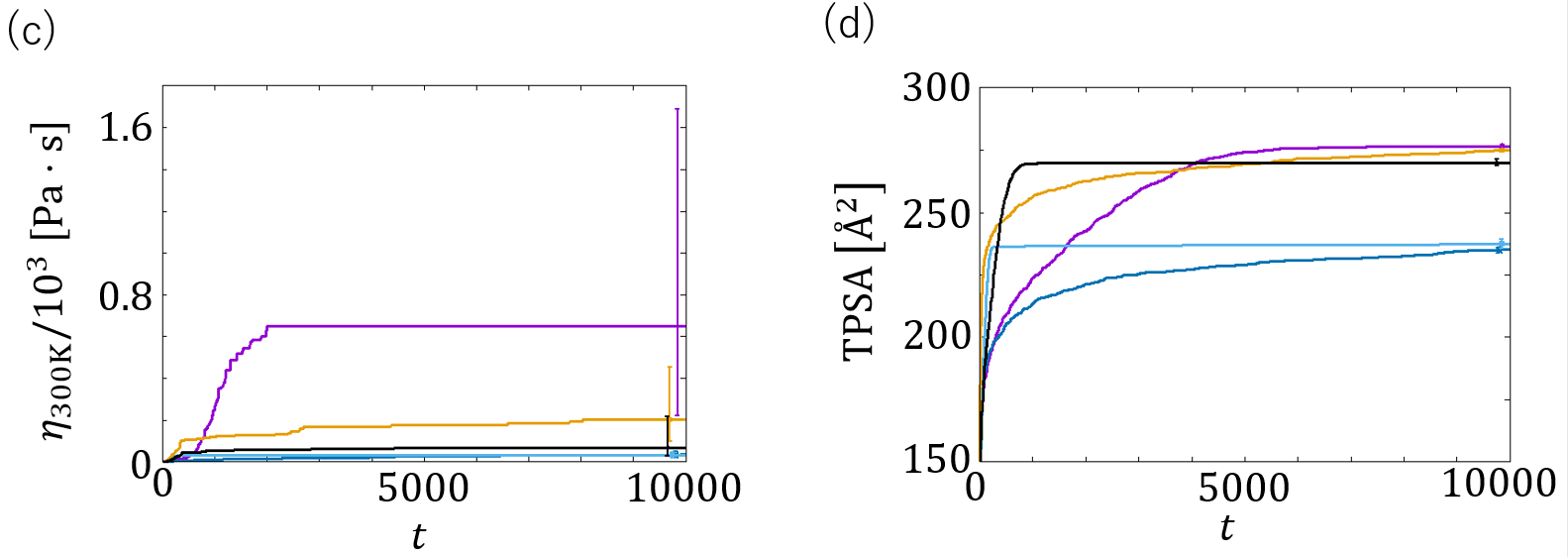}
  \end{tabular}
  \caption{
    Transition plots of the molecular discovery.
    (a) $T_\text{b}$, (b) $P_\text{c}$, (c) $\eta_{300\text{K}}$, and (d) TPSA.
    The other notations are the same as those in Figure \ref{fig1}.
    The colors have the following representation: purple, this work; sky blue, {\it spUCB};
    orange, {\it UCBE}; black, conventional UCB; and blue, random search.
    The error bars indicate the standard errors of the 100 independent runs
    in (a), (b), and (d).
    In (c), the quantiles are displayed instead of the standard errors.
    {\it ThresholdAscent} and {\it RobustUCBMax} cannot be implemented
    in MCTS because of the many hyperparameters involved.
  }\label{fig3}
\end{figure}

In the search of $T_\text{b}$ in Fig. \ref{fig3}(a),
the conventional UCB afforded the highest rewards at $t=10,000$.
This result is expected because the empirical formula of $T_\text{b}$
is a simple sum of the fragment parameters.
In such case, the optimal arm is almost equivalent to the arm with the best expectation reward.
This condition corresponds to the ``unfavorable'' case of synthetic problems.
In fact, the searches for other properties expressed
by simple summation in the Joback method afforded similar results.
For $t<2,500$, {\it spUCB} demonstrated the best performance.
This result is probably due to the exploitative hyperparameters
recommended in the original article \citep{schadd2008single}.
The conventional UCB with a smaller $c$ gave a similar transition plot.
Our algorithm exhibited the second-best performance in the late stage of the search process.
The results of {\it UCBE} and the random search were worse than the above.

In the searches of $P_\text{c}$, $\eta_{300\text{K}}$, and TPSA,
our algorithm demonstrated the best performance in the late stage.
In the early stage of the search processes, {\it UCBE} exchibited better and highly stable performance.
There are some different tendencies in these transition plots.
These differences are probably due to the differences in the population distributions of rewards.
For example, for $\eta_{300K}$, there are chemical structures
with enormously high rewards in the search space.
Our algorithm can find these structures with a high efficiency and success rate.
In contrast, for $\text{TPSA}$, the population distribution
probably has an upper bound near $290 \text{ \AA}^{2}$.
Even if such case, our algorithm worked well.
These results evidence the wide application range of our proposed algorithm.

\modifyA{
Samples of chemical structures with the highest $T_b$, $P_c$, $\eta_{300\text{K}}$,
or $\text{TPSA}$ of each run are shown in Figure \ref{chem}.
We have some understandings for high score molecules:
\begin{itemize}
\item Carboxyl groups are favorable for high $T_b$.
\item Alcohol, carboxyl, and halogen groups are favorable for high viscosity.
\item Polarized oxygen groups are favorable for high TPSA.
\end{itemize}
These understandings are consistent with chemical knowledge.
More complicated and highly optimized structures can be found
in our algorithm than in other algorithms.
\begin{figure}[htbp]
  \centering
  \includegraphics[bb= 0.000000 0.000000 816.864021 305.292614, width=14cm]{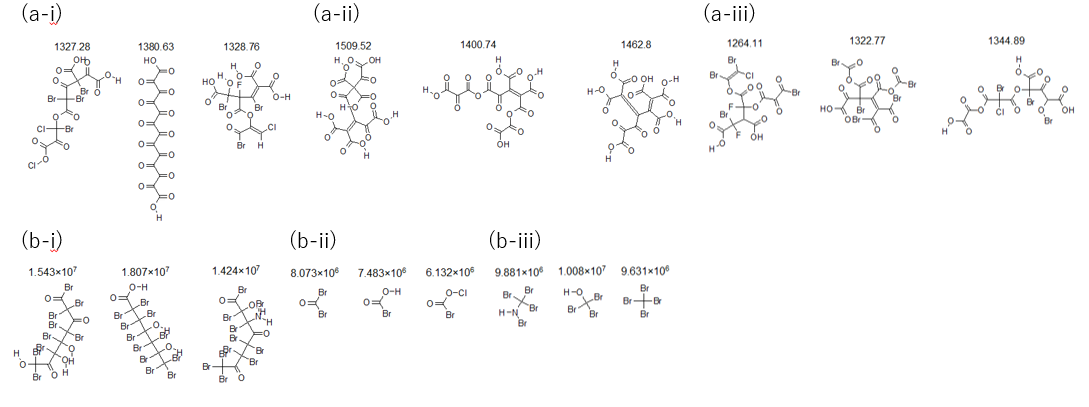}
  \includegraphics[bb= 0.000000 0.000000 820.614544 324.795336, width=14cm]{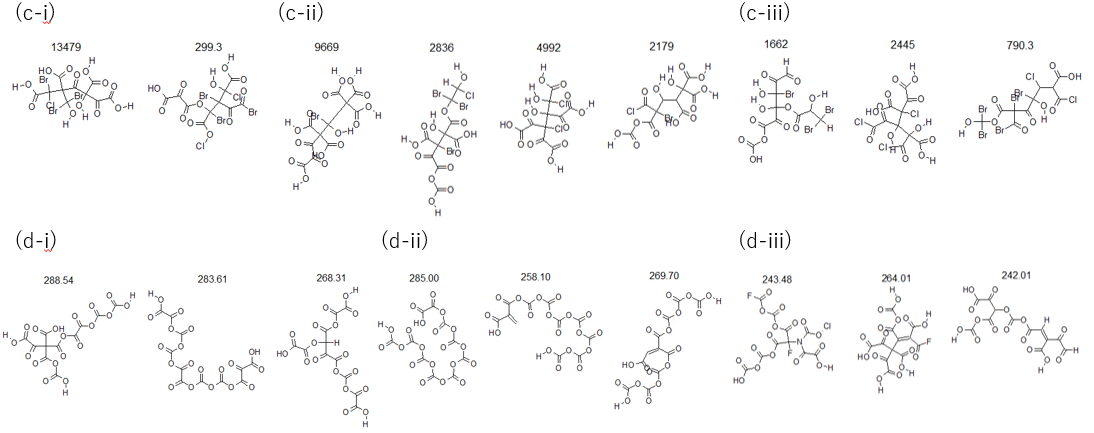}
  \caption{
    Sample of the chemical structures with the highest properties of each run.
    (a) $T_b$, (b) $P_c$, (c) $\eta_{300\text{K}}$, and (d) $\text{TPSA}$.
    (i) This work, (ii) UCB, and (iii) random search.
  }\label{chem}
\end{figure}

\section{Discussion}\label{sec7}
The numerical experiments in the previous section show
that a UCB of EI is suitable for the selection index for the MKB problem.
In this section, we discuss why that is so.
Our discussion would contain nonlogical arguments.
However, we believe that this discussion will help with future work.

\subsection{Subtleties of Extreme Regret}
For our discussion, we should mention the subtleties of extreme regret,
first pointed out by \citet{nishihara2016no}.
In this section, we review these subtleties.

The extreme regret was introduced by \citet{carpentier2014extreme}, defined as follows:
\begin{definition}[Carpentier's Regret]
  In the MKB problem, Carpentier's regret when $k(t), t\in[T]$
  are selected is defined as follows:
  \begin{equation}
    R^{k_*(t),k(t)}_\text{C}(T)
    \coloneqq \mathbb{E}\left[\max_{t\in [T]}r_{k_*(t)}(t)\right]
    - \mathbb{E}\left[\max_{t\in [T]}r_{k(t)}(t)\right],
  \end{equation}
  where $k_*(t), t\in[T]$ denotes an oracle policy.
  The asymptotically optimal policy should satisfy
  \begin{equation}\label{optimal}
    R^{k_*(t),k(t)}_\text{C}(T)=o\left(
      \mathbb{E}\left[\max_{t\in [T]}r_{k_*(t)}(t)\right]
    \right).
  \end{equation}
\end{definition}
A subtlety of Carpentier's regret is that
the regret asymptotically approaches $0$ for most policies in some settings.
For example, we consider all reward distributions of the arms have bounded support.
Then, any policy that selects each arm infinitely often achieves an asymtotically zero regret,
meaning that even the random search is asymptotically optimal in the setting.
It is a serious problem in previously proposed MKB algorithms
because most of them are essentially based on Carpentier's regret.

The problem of Carpentier's regret means that
this regret is unsuitable as an indicator of the asymptotically optimal policy.
To avoid this, \citet{nishihara2016no} defined an alternative regret:
\begin{definition}[Nishihara's Regret]
  In the MKB problem, Nishihara's regret when $k(t), t\in[T']$
  are selected is defined as follows:
  \begin{equation}
    R^{k_*(t),k(t)}_\text{N}(T)
    \coloneqq \frac{1}{T}\min_{T'\geq1}\left\{T':
      \mathbb{E}\left[\max_{t\in [T']}r_{k(t)}(t)\right]
      \geq \mathbb{E}\left[\max_{t\in [T]}r_{k_*(t)}(t)\right]
    \right\},
  \end{equation}
  where $k_*(t), t\in[T]$ denotes an oracle policy.
  The asymptotically optimal policy should satisfy
  \begin{equation}
    \limsup_{T\to\infty}R^{k_*(t),k(t)}_\text{N}(T) \leq 1.
  \end{equation}
\end{definition}
This regret works even when the reward distributions have finite supports.
However, \citet{nishihara2016no} showed that
there is a set of reward distributions such that 
$$\limsup_{T\to\infty}R^{k_*^\text{SA},k(t)}_\text{N}(T) \geq K$$
for any policy, where $k_*^\text{SA}$ is the selection of single-armed oracle
  defined in Definition \ref{singleArmedOracle}.
Namely, no policy is asymptotically optimal under Nishihara's regret.

Another subtlety exists in the definition of the oracle policy.
The previous works are essentially based on the single-armed oracle \citep{nishihara2016no} as follows:
\begin{definition}[Single-Armed Oracle]\label{singleArmedOracle}
  In the MKB problem, the single-armed oracle is the policy that plays the single arm
  \begin{equation}
    k_*^\text{SA}:={\argmax}_{k\in [K]}\mathbb{E}\left[\max_{t\in [T]}r_k(t)\right]
  \end{equation}
   over a time horizon $T$.
\end{definition}
However, this oracle gives different $k_*^\text{SA}$ depending on $T$.
This fact can be confirmed by the following example.
\begin{example}\label{ex1}
Consider the MKB problem with $K=3$.
Let the reward distributions of each arm be
$f_1(r)=\mathcal{N}(r;0,0.01)$, $f_2(r)=\mathcal{N}(r;-1,0.25)$, $f_3(r)=\mathcal{N}(r;-15,4)$.
Then, the single-armed oracle gives $k_*^\text{SA}=1$ if $T \leq 11$,
$k_*^\text{SA}=2$ if $1.4\times10^{11}\leq T \leq 5.4\times10^{13}$,
and $k_*^\text{SA}=3$ if $T \geq 3.9\times10^{202}$.
\end{example}
\begin{proof}
The expected maximum reward sampled from the $k$-th arm over a time horizon $T$
is bounded by
\begin{equation}
\mu_k+\frac{1}{\sqrt{\pi\ln 2}}\sigma_k\sqrt{\ln T} \leq
\mathbb{E}\left[\max_{t\in [T]}r_k(t)\right] \leq \mu_k+\sqrt{2}\sigma_k\sqrt{\ln T},
\end{equation}
where $\mu_k$ and $\sigma^2_k$ are the mean and variance
of the Gaussian reward distribution, $f_k(r)$, respectively \citep{kamath2015bounds}.
Then, the example is established. 
\end{proof}
The $T$-dependency of the single-armed oracle means
that the best arm cannot be determined without information on $T$ \citep{nishihara2016no}.
This raises a question about the regret analysis using the infinity limit of $T$.
In Example \ref{ex1}, arm $3$ should be selected most often
to achieve the asymptotically zero regret.
However, arm $1$ or $2$ is a more suitable choice when $T<5.4\times10^{13}$.
Because many applications cannot perform such a large number of trials,
the regret analysis result is impractical.
The $T$-dependency of the oracle is also seriously inconvenient in MCTS applications.
In an MCTS algorithm, $T$ is not given except for the root node.
Then, one cannot determine the best arm except for the root node
even if the reward distribution is known.
}

\modifyA{
\subsection{$T$-independent oracles and asymptotics of UCB approach}
An oracle independent of $T$ was also proposed  by \citet{nishihara2016no}.
\begin{definition}[Nishihara's Greedy Oracle]
  In the MKB problem, Nishihara's greedy oracle is the policy
  that plays the arm with the maximum EI.
  Namely, this oracle plays
  \begin{equation}
    k_*^\text{N}(\tau):=\argmax_{k\in [K]}\mathbb{E}\left[
      \max\left\{r_k(\tau),r_*^\text{max}(\tau-1)\right\}-r_*^\text{max}(\tau-1)
      | \{r_{k_*^\text{N}(t)}(t)\}_{t\in\tau-1}
    \right]
  \end{equation}
  at time $\tau$, where $r_*^\text{max}(\tau)\coloneqq\max_{t\in[\tau]}r_{k_*^\text{N}}(t)$.
\end{definition}
This oracle uses EI to avoid the dependency on $T$.
Therefore, in terms of Nishihara's greedy oracle, it is natural that we employ EI
to derive a $T$-independent MKB algorithm.  
Although \citet{nishihara2016no} did not analyze this oracle much,
we note that Nishihara's greedy oracle gives different $k_*^\text{N}(\tau)$
depending on the oracle value $r_*^\text{max}(\tau)$ instead of $T$, as shown in the following example:
\begin{example}\label{ex2}
Consider the MKB problem with $K=3$.
Let the reward distributions of each arm be
$f_1(r)=\mathcal{N}(r;0,1)$, $f_2(r)=\mathcal{N}(r;-2,2)$,
$f_3(r)=\mathcal{N}(r;-6,3)$.
Then, Nishihara's greedy oracle gives
$k_*^\text{N}(\tau)=1$ if $r_*^\text{max}(\tau-1) \leq 1.3$,
$k_*^\text{N}(\tau)=2$ if $7.0 \leq r_*^\text{max}(\tau-1) \leq 11.9$,
and $k_*^\text{N}(\tau)=3$ if $r_*^\text{max}(\tau-1) \geq 18.9$,
\end{example}
\begin{proof}
Because of Lemma \ref{EI2}, we should consider the integral of the survival function.
The survival function of $f_k(r)=\mathcal{N}(r;\mu_k,\sigma_k^2)$ is expressed as follows:
\begin{equation}
S(r)\coloneqq\int_r^\infty f_k(u)du
=\frac{1}{2}\erfc\left[\frac{r-\mu_k}{2\sigma_k}\right]
\end{equation}
The bounds of $\erfc(x),x>0$ are given by
\begin{equation}
  c\exp(-\beta x^2)<\erfc(x)<\exp(-x^2),
\end{equation}
where
\begin{equation}
  c=\sqrt{\frac{2e}{\pi}}\frac{\sqrt{\beta-1}}{\beta},
\end{equation}
and $\beta>1$ \citep{chiani2003new,chang2011chernoff}.
Then,
\begin{equation}
  \frac{c}{2}\int_y^\infty\exp(-\beta x^2)dx
  <\frac{1}{2}\int_y^\infty\erfc(x)dx
  <\frac{1}{2}\int_y^\infty\exp(-x^2)dx,
\end{equation}
where $y>0$.
Using the definition and bounds of $\erfc(x)$ again, we obtain
\begin{equation}
  \frac{\sqrt{\pi}c^2}{4\sqrt{\beta}}\exp(-\beta^2 y^2)
  <\frac{1}{2}\int_y^\infty\erfc(x)dx
  <\frac{\sqrt{\pi}}{4}\exp(-y^2).
\end{equation}
These bounds give the example.
\end{proof}
This dependency generates a subtlety in an adaptive case.
Consider one obtains $r^\text{max}(\tau)< 11.9$ under a selections
$\{k(t)\}_{t\in[\tau]}$ in Example \ref{ex2}.
A problem arises when the oracle value $r_*^\text{max}(\tau) > 18.9$ at that time.
In this case, the next selection of Nishihara's greedy oracle
differs from the selection with the maximum EI,
meaning that a policy simply approaching Nishihara's greedy oracle
is not always effective in the MKB problem.

The subtlety due to the dependence on the oracle value $r_*^\text{max}(\tau)$
is solved using the observed value $r^\text{max}(\tau)$ alternatively.
we define this oracle as follows:
\begin{definition}[Kikkawa's Greedy Oracle]
  In the MKB problem, let $k(t), t\in[\tau-1]$ be the previous selections.
  Then, Kikkawa's greedy oracle plays
  \begin{equation}
    k_*^\text{K}(\tau):=\argmax_{k\in [K]}\mathbb{E}\left[
      \max\left\{r_k(\tau),r^\text{max}(\tau-1)\right\}-r^\text{max}(\tau-1)
      | \{r_{k(t)}(t)\}_{t\in\tau-1}
    \right]
  \end{equation}
  at time $\tau$, where $r^\text{max}(\tau)\coloneqq\max_{t\in[\tau]}r_{k(t)}(t)$.
\end{definition}
This oracle is equivalent to Nishihara's greedy oracle
when all selections follow this oracle.
In addition, this oracle gives the arm that has the maximum EI
even when the non-oracle selections exist in $k(t),t\in[\tau-1]$.
The following proposition states that Kikkawa's greedy oracle
asymptotically approaches Nishihara's greedy oracle in terms of Carpentier's regret.
\begin{proposition}[Asymptotics of Kikkawa's Greedy Oracle]\label{asymptotics}
Let $k(t),t\in[T]$ contain $o(T)$ non-oracle selections
and other selections follow Kikkawa's greedy oracle.
Then,
  \begin{equation}
    R^{k_*^N(t),k(t)}_\text{C}(T)=o\left(
      \mathbb{E}\left[\max_{t\in [T]}r_{k_*^N(t)}(t)\right]
    \right),
  \end{equation}
when 
  \begin{equation}
    \mathbb{E}\left[\max_{t\in [T]}r_{k_*^N(t)}(t)\right] \geq O(1)
    \ \text{and}\  EI\left[k,T;G^\text{max}\right] \leq O(T^{-1}).
  \end{equation}
\end{proposition}
\begin{proof}
Let $n_k(T),k\in[K]$ and $n_k^\text{N}(T),k\in[K]$ be the numbers of the $k$-th arm selected
in $k(t),t\in[T]$ and $k_*^\text{N}(t),t\in[T]$, respectively.
Then, $\delta n \coloneqq \left|n_k^\text{N}(T)-n_k(T)\right|=o(T)$ is expected.%
\footnote{Pathological conditions may exist. However, we do not consider them.}
Therefore,
\begin{equation}
  \begin{split}
  \left|R^{k_*^\text{N}(t),k(t)}_\text{C}(T)\right|
  & = \left|\mathbb{E}\left[\max_{k\in[K]}\max_{n\in[n_k^\text{N}(T)]}r_k(n)\right]
    - \mathbb{E}\left[\max_{k\in[K]}\max_{n\in[n_k(T)]}r_k(n)\right]\right| \\
  & = \left|\mathbb{E}\left[\max_{k\in[K]}\left[\max\left\{
    \max_{n\in[n^\text{min}]}r_k(n), \max_{n\in[\delta n]}r_k(n+n^\text{min})
    \right\} - \max_{n\in[n^\text{min}]}r_k(n) \right]\right]\right| \\
  & \leq \sum_{k\in[K]}\left|\mathbb{E}\left[\max\left\{
    \max_{n\in[n^\text{min}]}r_k(n), \max_{n\in[o(T)]}r_k(n+n^\text{min})
    \right\} - \max_{n\in[n^\text{min}]}r_k(n) \right]\right| \\
  & \leq \sum_{k\in[K]}\left|o(T)EI\left[k,n^\text{min};G^\text{max}\right]\right| = o(1)
  \leq O\left(\mathbb{E}\left[\max_{t\in [T]}r_{k_*^\text{N}(t)}(t)\right]\right),
  \end{split}
\end{equation}
where $n^\text{min}\coloneqq\min\{n_k^\text{N}(T),n_k(T)\}$.
\end{proof}
The proposition states that $o(T)$ mistakes are allowed in the asymptotically optimal policy
under the condition related to the maximum value.
This condition can be satisfied by Gaussian distributions at least.

Our concept in Section \ref{sec4} can be obtained by simply substituting
the EI in Kikkawa's greedy oracle into its UCB.
Therefore, our conceptual algorithm is expected
to approach Kikkawa's greedy oracle for large $T$.
The number of non-oracle selections in the algorithm can be estimated as follows:
\begin{proposition}[Number of Non-Oracle Selections]\label{nonOracleSelections}
Consider the MKB problem.
Let $\overline{EI}\left[k,\mathcal{R}(t-1);G^\text{sum}\right],t\in[T]$
be an estimator of $EI\left[k,t;G^\text{max}\right]$.
Suppose a confidence interval between $EI\left[k,t;G^\text{max}\right]$
and $\overline{EI}\left[k,\mathcal{R}(t-1);G^\text{sum}\right]$ is known as
\begin{equation}
  \left|EI\left[k,t;G^\text{max}\right]
  -\overline{EI}\left[k,\mathcal{R}(t-1);G^\text{sum}\right]\right|\leq C(k,\alpha(t),n_k(t)),
\end{equation}
with confidence level $1-\alpha(t)$,
where $n_k(t),k\in[K]$ be the numbers of the $k$-th arm selected
under Algorithm \ref{maxSearch} with this confidence interval.
Then, the number of non-oracle selections becomes $o(T)$
when $\alpha(t)=o(1)$ and $C(k,\alpha(t),n_k(t))=o(\{t/n_k(t)\}^d)$
where $d>0$.
\end{proposition}
\begin{proof}
Consider the following events:
\begin{equation}\label{event1}
  \begin{split}
  &A_{k_*^\text{K}(t),t}: \\
  &\ \ \overline{EI}\left[k_*^\text{K}(t), \mathcal{R}(t-1);G^\text{sum}\right]
  \geq EI\left[
    k_*^\text{K}(t),t;G^\text{max}
  \right] - C\left(k_*^\text{K}(t),\alpha(t),n_{k_*^\text{K}(t)}(t)\right).
  \end{split}
\end{equation}
\begin{equation}\label{event2}
  A_{\kappa(t),t}:
  \overline{EI}\left[\kappa(t),\mathcal{R}(t-1);G^\text{sum}\right]
  \leq EI\left[\kappa(t),t;G^\text{max}\right] + C\left(
    \kappa(t),\alpha(t),n_{\kappa(t)}(t)
  \right),
\end{equation}
where $\kappa(t)\in[K]/k_*^\text{K}(t)$.
Then, the number of complementary cases can easily be counted as follows:
\begin{equation}\label{complementary}
  \sum_{t\in[T]}\mathbb{E}\left[\mathbbm{1}\left[
    \bigcup_{k\in[K]}A_{k,t}^c
  \right]\right] \leq K\sum_{t\in[T]}\alpha(t) = o(T).
\end{equation}
Conversely, in the case of all $A_{k,t}$ established,
the non-oracle arm is selected when
\begin{equation}\label{nonOracle}
  \begin{split}
  \overline{EI}\left[\kappa(t),\mathcal{R}(t-1);G^\text{sum}\right]
    & + C\left(\kappa(t),\alpha(t),n_{\kappa(t)}(t)\right) \\
  & \geq \overline{EI}\left[
    k_*^\text{K}(t),\mathcal{R}(t-1);G^\text{sum}
  \right] + C\left(k_*^\text{K}(t),\alpha(t),n_{\kappa(t)}(t)\right),
  \end{split}
\end{equation}
for any $\kappa(t)$.
Then,
\begin{equation}
\begin{split}
  EI\left[k_*^\text{K}(t),t;G^\text{max}\right]
  & \leq \overline{EI}\left[
    k_*^\text{K}(t),\mathcal{R}(t-1);G^\text{sum}
  \right] + C(k_*^\text{AG}(t),\alpha(t),n_{k_*^\text{K}(t)}(t)) \\
  & \leq \overline{EI}\left[\kappa(t),\mathcal{R}(t-1);G^\text{sum}\right]
  + C(\kappa(t),\alpha(t),n_{\kappa(t)}(t)) \\
  & \leq EI\left[\kappa(t),t;G^\text{max}\right] + 2C(\kappa(t),\alpha(t),n_{\kappa(t)}(t)).
\end{split}
\end{equation}
Equations (\ref{event1}), (\ref{event2}), and (\ref{nonOracle})
are used in the first, second, and third inequalities, respectively.
Then, solving for $n_{\kappa(t)}(t)$ using 
$C(k,\alpha(t),n_{\kappa(t)}(t))=o(\{t/n_k(t)\}^d)$,
we obtain
\begin{equation}\label{nk}
  n_{\kappa(t)}(t) \leq 2^{1/d} \Delta_{\kappa(t)}(t)^{-1/d}o(t),
\end{equation}
where
\begin{equation}
\Delta_k(t)=EI\left[
  k_*^\text{K}(t),t;G^\text{max}\right]-EI\left[k,t;G^\text{max}
\right].
\end{equation}
Using Equations (\ref{complementary}) and (\ref{nk}), we obtain the proposition.
\end{proof}
This proposition means that Algorithm \ref{maxSearch} 
asymptotically approaches Kikkawa's greedy oracle
in terms of the number of non-oracle selections.
Considering Proposition \ref{asymptotics}, Algorithm \ref{maxSearch}
is also an asymptotically optimal policy
in terms of Nishihara's greedy oracle and Carpentier's regret.

Notably, the proof of Proposition \ref{nonOracleSelections}
is an analog of the proof for the conventional bandit problem
\citep{auer2002finite,Jamieson2018lecture}
except for the optimal arm depending on the time $t$.
This treatment can be allowed because the selections by Kikkawa's greedy oracle
correspond to the arms with the maximum EI for any $t$.
This feature of Kikkawa's greedy oracle is strong. 
We are sure that several other proofs for the conventional bandit problem
will be established formally in the MKB problem using Kikkawa's greedy oracle.
}

\section{Conclusion}\label{sec8}
Here, we proposed an MKB algorithm and applied it
to synthetic problems and molecular-design demonstrations
using MCTS for materials discovery.
The proposed algorithm only uses one hyperparameter and is easy to implement for MCTS.
This feature gives the proposed algorithm an advantage over other MKB algorithms
\citep{carpentier2014extreme, david2016pac, streeter2006simple,
achab2017max, streeter2006asymptotically},
and enables its application for materials discovery.
In fact, to the best of our knowledge,
this is the first case where the MKB algorithm is actually employed for materials discovery.
The performance of the proposed algorithm was examined
using the synthetic problems and the molecular-structure optimizations.
The experimental results demonstrated that the proposed algorithm
found the maximum reward more efficiently than other algorithms
when the optimal arm could not be determined only based on the expectation reward.
In real molecular designs, most of the molecular properties would have a high complexity;
thus, we believe that the proposed algorithm is useful for these tasks.

\modifyA{In the theoretical aspect, we mainly contribute in two aspects.
One is the proof of the effectiveness of the use of a UCB of EI.
The proof result has wide flexibility and using this,
other algorithms can be proposed with other assumptions for variables,
which will be addressed in future work.
The other is the proposal of Kikkawa's greedy oracle.
Using the proposed oracle, we can avoid many of the subtleties of the MKB problem.

Although we do not treat in this study,
}%
heuristics to reduce the required trials are also important for actual use.
\modifyA{For example,} the combination of the algorithm with {\it UCBE}
may present a strategy for reducing the required trials.
In addition, a combination with a supervised learning approach holds significant promise.
The application of the proposed algorithm and the derivation into other areas
are also important aspects that require further investigation.

\acks{
  We wish to thank Dr. Ryosuke Jinnouchi in TCRDL for reviewing our early draft.
}

\appendix
\section{Proof of Theorem \ref{Bernstein}}\label{app1}
Let $X_{i}$ be independent random variables drawn
from the same sub-exponential $g(X)$ with parameter $b > 0$.
Then,
\begin{equation}
  \begin{split}
    &\mathbb{P}\left\{ \left|
    \frac{1}{n} \sum_{i=1}^{n} X_{i} - \mathbb{E}_{g(X)}\left[ X \right]
    \right| \geq u \right\}\\
    &= \mathbb{P}\left\{
    \frac{1}{n}\sum_{i=1}^nX_i - \mathbb{E}_{g(X)} \geq u \; \text{or} \;
    \mathbb{E}_{g(X)} - \mathbb{P}\left\{\frac{1}{n}\sum_{i=1}^nX_i\right\} \geq u
    \right\},
  \end{split}
\end{equation}
where $u > 0$.
Therefore, in the former case,
\begin{equation}
  \begin{split}
    &\mathbb{P}\left\{
    \frac{1}{n} \sum_{i=1}^{n} X_{i} - \mathbb{E}_{g(X)}\left[ X \right] \geq u
    \right\}\\
    &= \mathbb{P}\left\{
    \exp{ \left( \frac{\lambda}{n} \sum_{i=1}^{n} X_{i} \right) }
    \geq \exp{ \left[ \lambda (u+\mathbb{E}_{g(X)}\left[ X \right]) \right] }
    \right\}\\
    &\leq \exp{ \left[- \lambda (u+\mathbb{E}_{g(X)}\left[ X \right]) \right] }
    \mathbb{E}\left[
      \exp{ \left( \frac{\lambda}{n} \sum_{i=1}^{n} X_{i} \right) }
      \right], \; \; \; \langle{\rm Markov's} \; {\rm inequality}\rangle
  \end{split}
\end{equation}
where $\lambda>0$ is an arbitrary parameter.
Since the random variables $X_{i}$ are independent of each other,
their moment-generating function can be separated.
Thus, we obtain
\begin{equation}
  \begin{split}
    &\mathbb{P}\left\{
    \frac{1}{n} \sum_{i=1}^{n} X_{i} - \mathbb{E}_{g(X)}\left[ X \right] \geq u
    \right\}\\
    &\leq \exp{ \left[ -\lambda (u+\mathbb{E}_{g(X)}\left[ X \right]) \right] }
    \left\{
    \mathbb{E}_{g(X)}\left[ \exp{ \left( \frac{\lambda X}{n} \right) } \right]
    \right\}^{n}\\
    &= \exp{ \left[- \lambda (u+\mathbb{E}_{g(X)}\left[ X \right]) \right] }
    \left(
    1 + \frac{ \lambda \mathbb{E}_{g(X)}\left[ X \right] }{n}
    + \sum_{p=2}^{\infty} \frac{ \lambda^{p} \mathbb{E}_{g(X)}\left[X^{p}\right]}{n^{p}p!}
    \right)^{n}\\
    &\leq \exp{ \left[ - \lambda u + \sum_{p=2}^{\infty}
        \frac{\lambda^{p} \mathbb{E}_{g(X)}\left[X^{p}\right]}{n^{p-1}p!}
        \right] } \; \; \; \langle{\rm Since}\, 1 + x \leq \exp{x}\rangle\\
    &= \exp{ \left[
        - \lambda u + \sum_{p=2}^{\infty} \frac{\lambda^{p}}{n^{p-1}p!}
        \int_{0}^{\infty} \mathbb{P}_{g(X)}\{ X^{p} \geq u \} \; du
        \right] } \; \; \; \langle{\rm Integral} \; {\rm identity}\rangle\\
    &= \exp{ \left[
        - \lambda u + \sum_{p=2}^{\infty} \frac{\lambda^{p}}{n^{p-1}p!}
        \int_{0}^{\infty} \mathbb{P}_{g(X)}\{ X \geq bv \}pb^{p}v^{p-1} \; dv
        \right] } \; \; \; \langle{\rm Replace} \; u \; {\rm with} \; b^{p}v^{p}\rangle\\
    &\leq \exp{ \left[ - \lambda u + 2 \sum_{p=2}^{\infty}
        \frac{\lambda^{p}b^{p}}{n^{p-1}(p-1)!} \int_{0}^{\infty} e^{-v} v^{p-1} \; dv
        \right] }. \; \; \; \langle{\rm Sub\mathchar`-exponential}\rangle
  \end{split}
\end{equation}
The above integral corresponds to the Gamma function. Therefore,
\begin{equation}
  \begin{split}
    \mathbb{P}\left\{
    \frac{1}{n} \sum_{i=1}^{n} X_{i} - \mathbb{E}_{g(X)}\left[ X \right] \geq u
    \right\}
    &\leq \exp{ \left[ -\lambda u + 2 \sum_{p=2}^{\infty}
        \frac{ b^{p} \lambda^{p} \Gamma(p) }{ n^{p-1} (p-1)! } \right] }\\
    &= \exp{ \left[ -\lambda u + 2 \sum_{p=2}^{\infty}
        \frac{ b^{p} \lambda^{p} }{ n^{p-1} } \right] }\\
    &= \exp{ \left[ -\lambda u + \frac{ 2 b^{2} \lambda^{2} }{ n - b\lambda } \right] },
  \end{split}
\end{equation}
where $\left| b\lambda/n \right| < 1$.
Replacing $ n - b \lambda $ with $ \xi_{+} $, we obtain
\begin{equation}
  \mathbb{P}\left\{
  \frac{1}{n} \sum_{i=1}^{n} X_{i} - \mathbb{E}_{g(X)}\left[ X \right] \geq u
  \right\} \leq \exp{ \left[
      \left( \frac{u}{b} + 2 \right) \xi_{+} + \frac{2n^{2}}{\xi_{+}} - \frac{nu}{b} - 4n \right] },
\end{equation}
where $ 0 < \xi_{+} < 2n $.
Hence, the optimized $ \xi_{+} $ is
\begin{equation}
  \xi_{+} = \sqrt{ \frac{2n^{2}b}{2b + u} }.
\end{equation}
Then,
\begin{equation}\label{eq-a1}
  \mathbb{P}\left\{
  \frac{1}{n} \sum_{i=1}^{n} X_{i} - \mathbb{E}_{g(X)}\left[ X \right] \geq u \right\} \leq \exp{ \left( 2\sqrt{2} n w_{+} - n w_{+}^{2} - 2n \right) },
\end{equation}
where $ w_{+} \coloneqq \sqrt{ ub^{-1} + 2 } \geq \sqrt{2} $.
Moreover, using the same approach, we obtain
\begin{equation}
  \mathbb{P}\left\{
  \mathbb{E}_{g(X)}\left[ X \right] - \frac{1}{n} \sum_{i=1}^{n} X_{i} \geq u \right\} \leq \exp{ \left[ \left(-\frac{u}{b} + 2 \right) \xi_{-} + \frac{2n^{2}}{\xi_{-}} + \frac{nu}{b} - 4n \right] },
\end{equation}
where $ \xi_{-} \coloneqq n + b \lambda $ and $ 0 < \xi_{-} < 2n $.
Hence, the optimized $ \xi_{-} $ is
\begin{equation}
  \xi_{-} = \left\{ \begin{array}{ll}
    \sqrt{ 2n^{2}b/(2b-u) } & 0 \leq u < 3b/2\\
    2n - \epsilon & 3b/2 \leq u,
  \end{array} \right.
\end{equation}
where $\epsilon$ is an infinitesimal.
Then, we obtain
\begin{equation}\label{eq-a2}
  \mathbb{P}\left\{
  \mathbb{E}_{g(X)}\left[ X \right] - \frac{1}{n} \sum_{i=1}^{n} X_{i} \geq u
  \right\} \leq \left\{
  \begin{array}{ll}
    \exp{ (2\sqrt{2} nw_{-} - n w_{-}^{2} - 2n) } & 0 \leq u < 3b/2\\
    \exp{ \left[
        \left(- ub^{-1} + 1 \right) n + O(\epsilon)
        \right] } & 3b/2 \leq u,
  \end{array} \right.
\end{equation}
where $w_{-} \coloneqq \sqrt{ -ub^{-1} + 2 } > 1/\sqrt{2}$.
Equations \ref{eq-a1} and \ref{eq-a2} show that
$\mathbb{E}\left [ X \right]$ is bounded at a confidence level of $1-\alpha>0$ as follows:
\begin{equation}
  \frac{1}{n} \sum_{i=1}^{n} X_{i} - u_{+}^{*} \leq \mathbb{E}\left[ X \right] \leq \frac{1}{n} \sum_{i=1}^{n} X_{i} + u_{-}^{*},
\end{equation}
where
\begin{equation}\label{eq-a3}
  \alpha = \exp{ \left[ 2\sqrt{2} n w_{+}^{*} - n (w_{+}^{*})^{2} - 2n \right] },
\end{equation}
when $ u_{+}^{*} \geq 0 $, and
\begin{equation}\label{eq-a4}
  \alpha = \left\{ \begin{array}{ll}
    \exp{ \left[ 2\sqrt{2}nw_{-}^{*} - n (w_{-}^{*})^{2} - 2n \right] }
    & 0 \leq u_{-}^{*} < 3b/2\\
    \exp{ \left[ \left(- u_{-}^{*}b^{-1} + 1 \right)n \right] }
    & 3b/2 \leq u_{-}^{*},
  \end{array} \right.
\end{equation}
where $w_{\pm}^{*} \coloneqq \sqrt{ \pm u_{\pm}^{*} b^{-1} + 2 }$ (double sign in the same order).
From Eq. \ref{eq-a3}, we obtain
\begin{equation}
  w_{+}^{*} = \sqrt{2} + \beta,
\end{equation}
and
\begin{equation}
  u_{+}^{*} = (\beta^{2} + 2 \sqrt{2} \beta)b,
\end{equation}
where $\beta \coloneqq \sqrt{ -\ln \alpha/n }$.
In addition, from Eq. \ref{eq-a4},
\begin{equation}
  u_{-}^{*}= \left\{ \begin{array}{ll}
    (-\beta^{2} + 2 \sqrt{2}\beta)b & 0 \leq \beta < 1/\sqrt{2}\\
    (\beta^{2} + 1)b & 1/\sqrt{2} \leq \beta.
  \end{array} \right.
\end{equation}
The theorem follows from Eqs. \ref{eq-a1} and \ref{eq-a2}.

\section{Compared Algorithms}\label{app2}
We compared our algorithm with Algorithms \ref{ThresholdAscent}-\ref{RandomSearch}.
We employed the following hyperparameters and applied some modifications for the implementation.
In {\it ThresholdAscent}, the hyper-parameters were set to $s=100$ and $\delta=2\ln\nu$.
We used the reward ranking instead of the iteration
used in the original code \citep{streeter2006simple}.
In {\it RobustUCBMax}, we set $u=r^{100\text{-th}}$, $v=(r^{\max}-u)^{1+\epsilon}$,
and $\epsilon=0.4$, according to the original paper \citep{achab2017max}.
Although the original paper employed the robust UCB with the truncated mean estimator, 
we used a simple version of the robust UCB \citep{bubeck2013bandits}.
In {\it spUCB}, $c=0.1$ and $D=32$ are used as the hyper-parameters.
These values are recommended in the original paper \citep{schadd2008single}.
In {\it UCBE} \citep{audibert2010best} and the conventional UCB \citep{auer2002finite},
we used $c=1$ as the hyperparameter.
In some algorithms, we estimated the variance parameter, $\sigma$,
as the sample variance of the first $P=10$ random searches.
\begin{algorithm}[ht]
  \caption{ThresholdAscent}\label{ThresholdAscent}
  \begin{algorithmic}[1]
    \REQUIRE
    number of arms $K$,
    time horizon $T$,
    the $s$-th maximum of observed reward $r^{s\text{-th}}$
    number of times the $k$-th arm is selected $n_{k}$,
    the $i$-th reward from the $k$-th arm $r_{k,i}$,
    and hyper-parameter $\delta$.
    \ENSURE selected arm index $\hat{k}$.
    \FOR{\textbf{each} $k\in[K]$}
    \IF{$n_k=0$ or $\nu<2$}
    \STATE $z_{k} \leftarrow \infty$
    \ELSE
    \STATE $S_k=\sum_{i\in[n_k]}\mathbbm{1}[r_{k,i}>r^{s\text{-th}}]$
    \STATE $\alpha \leftarrow \ln(2TK/\delta)$
    \STATE $z_{k} \leftarrow S_k/n_k + (\alpha+\sqrt{\alpha(2S_k+\alpha)})/n_k$
    \ENDIF
    \ENDFOR
    \STATE $\hat{k} \leftarrow \argmax_{k\in[K]}z_{k}$
    \RETURN $\hat{k}$
  \end{algorithmic}
\end{algorithm}
\begin{algorithm}[ht]
  \caption{RobustUCBMax}
  \begin{algorithmic}[1]
    \REQUIRE
    number of arms $K$,
    number of times the $k$-th arm is selected $n_{k}$,
    the $i$-th reward from the $k$-th arm $r_{k,i}$,
    and hyper-parameters $u$, $v$, and $\epsilon$.
    \ENSURE selected arm index $\hat{k}$.
    \STATE $\nu=\sum_{k\in[K]}n_k$
    \FOR{\textbf{each} $k\in[K]$}
    \IF{$n_k=0$ or $\nu<2$}
    \STATE $z_{k} \leftarrow \infty$
    \ELSE
    \STATE $S_k=\sum_{i\in[n_k]}r_{k,i}\mathbbm{1}[r_{k,i}>u]$
    \STATE $z_{k} \leftarrow S_k/n_k \
      + 4v^{1/(1+\epsilon)}(2\ln\nu/n_k)^{\epsilon/(1+\epsilon)}$
    \ENDIF
    \ENDFOR
    \STATE $\hat{k} \leftarrow \argmax_{k\in[K]}z_{k}$
    \RETURN $\hat{k}$
  \end{algorithmic}
\end{algorithm}
\begin{algorithm}[ht]
  \caption{spUCB}
  \begin{algorithmic}[1]
    \REQUIRE
    number of arms $K$,
    current time $\tau$,
    number of times the $k$-th arm is selected $n_{k}$,
    sum of the rewards obtained from the $k$-th arm $R_{k}$,
    sum of square rewards obtained from the $k$-th arm \modifyA{$Q_k$},
    and sample variance obtained from the first $P$ trials $\sigma$.
    \ENSURE selected arm index $\hat{k}$.
    \IF{$\tau \leq P$}
    \STATE $\hat{k} \leftarrow RandomSearch(K)$
    \ELSE
    \STATE $\nu=\sum_{k\in[K]}n_k$
    \FOR{\textbf{each} $k\in[K]$}
    \IF{$n_k=0$ or $\nu<2$}
    \STATE $z_{k} \leftarrow \infty$
    \ELSE
    \STATE $m_k \leftarrow R_k/n_k$
    \STATE $z_{k} \leftarrow m_k + c\sigma\sqrt{\ln\nu/n_k} \
      + \sqrt{\frac{Q_k-n_km_k^2+D}{n_k}}$
    \ENDIF
    \ENDFOR
    \STATE $\hat{k} \leftarrow \argmax_{k\in[K]}z_{k}$
    \ENDIF
    \RETURN $\hat{k}$
  \end{algorithmic}
\end{algorithm}
\begin{algorithm}[ht]
  \caption{UCBE}
  \begin{algorithmic}[1]
    \REQUIRE
    number of arms $K$,
    current time $\tau$,
    number of times the $k$-th arm is selected $n_{k}$,
    sum of the rewards obtained from the $k$-th arm $R_{k}$,
    and sample variance obtained from the first $P$ trials $\sigma$.
    \ENSURE selected arm index $\hat{k}$.
    \IF{$\tau \leq P$}
    \STATE $\hat{k} \leftarrow RandomSearch(K)$
    \ELSE
    \STATE $\nu=\sum_{k\in[K]}n_k$
    \FOR{\textbf{each} $k\in[K]$}
    \IF{$n_k=0$ or $\nu<2$}
    \STATE $z_{k} \leftarrow \infty$
    \ELSE
    \STATE $z_{k} \leftarrow R_k/n_k + c\sigma\sqrt{\nu/n_k}$
    \ENDIF
    \ENDFOR
    \STATE $\hat{k} \leftarrow \argmax_{k\in[K]}z_{k}$
    \ENDIF
    \RETURN $\hat{k}$
  \end{algorithmic}
\end{algorithm}
\begin{algorithm}[ht]
  \caption{UCB}
  \begin{algorithmic}[1]
    \REQUIRE
    number of arms $K$,
    current time $\tau$,
    number of times the $k$-th arm is selected $n_{k}$,
    sum of the rewards obtained from the $k$-th arm $R_{k}$,
    and sample variance obtained from the first $P$ trials $\sigma$.
    \ENSURE selected arm index $\hat{k}$.
    \IF{$\tau \leq P$}
    \STATE $\hat{k} \leftarrow RandomSearch(K)$
    \ELSE
    \STATE $\nu=\sum_{k\in[K]}n_k$
    \FOR{\textbf{each} $k\in[K]$}
    \IF{$n_k=0$ or $\nu<2$}
    \STATE $z_{k} \leftarrow \infty$
    \ELSE
    \STATE $z_{k} \leftarrow R_k/n_k + c\sigma\sqrt{\ln\nu/n_k}$
    \ENDIF
    \ENDFOR
    \STATE $\hat{k} \leftarrow \argmax_{k\in[K]}z_{k}$
    \ENDIF
    \RETURN $\hat{k}$
  \end{algorithmic}
\end{algorithm}
\begin{algorithm}[ht]
  \caption{RandomSearch}\label{RandomSearch}
  \begin{algorithmic}[1]
    \REQUIRE
    number of arms $K$.
    \ENSURE selected arm index $\hat{k}$.
    \STATE $\hat{k} \leftarrow random(K)$
    \COMMENT{randomly select any of $1,...,K$.}
    \RETURN $\hat{k}$
  \end{algorithmic}
\end{algorithm}

\clearpage

\bibliography{main}

\end{document}